\documentclass[11pt,draftcls,onecolumn]{IEEEtran}
\usepackage{amsmath, amsthm, amssymb}
\usepackage{latexsym}
\usepackage{exscale}
\usepackage{fontenc}
\usepackage{graphicx}
\usepackage{epsfig}
\usepackage{nicefrac}
\usepackage{mathptmx}
\usepackage{xcolor,epic,eepic}
\usepackage[ruled,lined,linesnumbered]{algorithm2e}
\usepackage{subfigure}
\usepackage{fancyhdr}
\usepackage{hyperref}
\usepackage{multirow}
\usepackage{setspace}
\usepackage{authblk}
\usepackage{mathdots}

\begin{document}

\renewcommand{\baselinestretch}{1.1}
\newcommand{\nin}{\in\!\!\!/}
\newcommand{\rb}{\rangle}
\newcommand{\lb}{\langle}
\newcommand{\R}{{\bf R}}
\newcommand{\thTh}{{\theta \in \Theta}}
\newcommand{\Dip}{{D}}
\newcommand{\cG}{{\cal G}}
\newcommand{\cQ}{{\cal Q}}
\newcommand{\curvelet}{{c}}
\newcommand{\orient}{{\theta}}
\newcommand{\ERC}{{\rm ERC}}
\newcommand{\DU}{{{\cal D}_U}}
\newcommand{\ipidx}{{p}}
\newcommand{\ipatom}{{g}}
\newcommand{\Ga}{{\Gamma}}
\newcommand{\pGa}{{p \in \Gamma}}
\newcommand{\lu}{{\bf l}^1}
\newcommand{\Ld}{{\bf L}^2}
\newcommand{\lz}{{\bf l}^0}
\newcommand{\ld}{{\bf l}^2}
\newcommand{\V}{{\bf V}}
\newcommand{\cL}{{\cal L}}
\newcommand{\W}{{\bf W}}
\newcommand{\La}{{\Lambda}}
\newcommand{\tLa}{{\tilde \Lambda}}
\newcommand{\pLa}{{p \in \Lambda}}
\newcommand{\qO}{{q \in \Omega}}
\newcommand{\ptLa}{{p \in \tilde \Lambda}}
\newcommand{\supp}{{\Lambda}}
\newcommand{\om}{{\omega}}
\newcommand{\mq}{{m,q}}
\newcommand{\jn}{{j,n}}
\newcommand{\opt}{{\tilde}}
\newcommand{\IPop}{{U}}
\newcommand{\CC}{{\bf C}}
\newcommand{\Z}{{\bf Z}}
\newcommand{\Reg}{\Lambda}
\newtheorem{conjecture}{Conjecture}
\newtheorem{definition}{Definition}
\newtheorem{theorem}{Theorem}
\newtheorem{lemma}{Lemma}
\newtheorem{corollary}{Corollary}
\newtheorem{prop}{Proposition}
\providecommand{\argmin}{\mathop{\textup{argmin}}}
\def\truelabel{\label}
\newcommand{\Span}{\mathop{\textup{span}}}
\newcommand{\pen}{\text{pen}}
\newcommand{\ud}{\textup{d}}
\newcommand{\D}{\mathcal{D}}
\newcommand{\Mg}{\mathcal{M}_{\gamma}}
\newcommand{\Tde}{T^{\frac{2\alpha}{\alpha+1}}}
\newcommand{\B}{\mathcal{B}}
\newcommand{\wor}{k}
\newcommand{\Cal}{\mathbf{C}^{\alpha}}
\newcommand{\eqdef}{\overset{.def}{=}}
\newcommand{\tga}{\tilde g}
\newcommand{\ldeuxj}{l^2_j}
\newcommand{\ga}{g}
\newcommand{\nogeom}{\Xi}
\newcommand{\tS}{\tilde S}
\newcommand{\tb}{\tilde b}
\newcommand{\ldeux}{l^2}
\newcommand {\ImU} {{\bf ImU}}
\newcommand{\Proba}{\mathbb{P}}
\newcommand{\C} {{\bf C}}
\newcommand{\gammageom}{\upsilon}
\newcommand{\Ch}[1]{{\bf Ch: #1}}

\newcommand{\ba}{\mathbf{a}}
\newcommand{\bbf}{\mathbf{f}}
\newcommand{\bx}{\mathbf{x}}
\newcommand{\by}{\mathbf{y}}
\newcommand{\bw}{\mathbf{w}}
\newcommand{\bz}{\mathbf{z}}

\newcommand{\bzero}{\mathbf{0}}

\newcommand{\bA}{\mathbf{A}}
\newcommand{\bB}{\mathbf{B}}
\newcommand{\bC}{\mathbf{C}}
\newcommand{\bD}{\mathbf{D}}
\newcommand{\bI}{\mathbf{I}}
\newcommand{\bM}{\mathbf{M}}
\newcommand{\bR}{\mathbf{R}}
\newcommand{\bS}{\mathbf{S}}
\newcommand{\bT}{\mathbf{T}}
\newcommand{\bU}{\mathbf{U}}
\newcommand{\bW}{\mathbf{W}}

\newcommand{\sample}{\Phi}
\newcommand{\dict}{\Psi}
\newcommand{\expect}{E}
\newcommand{\nsp}{\mathrm{Null}}

\linespread{0.97}

\title{Statistical Compressed Sensing of Gaussian Mixture Models}

\author{Guoshen \textsc{Yu}}
\author{Guillermo \textsc{Sapiro}}
\affil{ECE, University of Minnesota, Minneapolis, Minnesota, 55414, USA}

\maketitle
\vspace{-7ex}
\begin{center}
\today
\end{center}

\vspace{5ex}
\begin{abstract}
A novel framework of compressed sensing, namely \textit{statistical compressed sensing} (SCS), that aims at efficiently sampling a \textit{collection} of signals that follow a statistical distribution, and achieving accurate reconstruction \textit{on average}, is introduced. SCS based on Gaussian models is investigated in depth. 
For signals that follow a single Gaussian model, with Gaussian or Bernoulli sensing matrices of $\mathcal{O}(k)$ measurements, considerably smaller than the $\mathcal{O}(k \log(N/k))$ required by conventional CS based on sparse models, where $N$ is the signal dimension, and with an optimal decoder implemented via linear filtering, significantly faster than the pursuit decoders applied in conventional CS, the error of SCS is shown tightly upper bounded by a constant times the best $k$-term approximation error, with overwhelming probability. The failure probability is also significantly smaller than that of conventional sparsity-oriented CS.  Stronger yet simpler results further show that for \textit{any} sensing matrix, the error of Gaussian SCS is upper bounded by a constant times the best $k$-term approximation with probability one, and the bound constant can be efficiently calculated. For Gaussian mixture models (GMMs), that assume multiple Gaussian distributions and that each signal follows one of them with an unknown index, a piecewise linear estimator is introduced to decode SCS. The accuracy of model selection, at the heart of the piecewise linear decoder, is analyzed in terms of the properties of the Gaussian distributions and the number of sensing measurements. A maximum a posteriori expectation-maximization algorithm that iteratively estimates the Gaussian models parameters, the signals model selection,  and decodes the signals, is presented for GMM-based SCS. In real image sensing applications, GMM-based SCS is shown to lead to improved results compared to conventional CS, at a considerably lower computational cost. 
\end{abstract}

\section{Introduction}

Compressed sensing (CS) aims at achieving accurate signal reconstruction while sampling signals at a low sampling rate, typically far smaller than that of Nyquist/Shannon. Let $\bx \in \mathbb{R}^N$ be a signal of interest, $\Phi \in \mathbb{R}^{M \times N}$ a \textit{non-adaptive} sensing matrix (\textit{encoder}), consisting of $M \ll N$ measurements, $\by = \Phi \bx \in \mathbb{R}^M$ a measured signal, and $\Delta$ a \textit{decoder} used to reconstruct $\bx$ from $\Phi \bx$. CS develops encoder-decoder pairs $(\Phi, \Delta)$ such that a small reconstruction error $\bx - \Delta(\Phi \bx)$ can be achieved. 

Reconstructing $\bx$ from $\Phi \bx$ is an ill-posed problem whose solution requires some prior information  on the signal. Instead of the frequency band-limit signal model assumed in classic Shannon sampling theory, conventional CS adopts a sparse signal model, i.e., there exists a dictionary, typically an orthogonal basis $\dict \in \mathbb{R}^{N \times N}$, a linear combination of whose columns generates an accurate approximation of the signal, $\bx \approx \dict \ba$, the coefficients $\ba[m]$, $1 \leq m \leq N$, having their amplitude decay fast after being sorted. For signals following the sparse model, it has been shown that using some random sensing matrices such as Gaussian and Bernoulli matrices $\Phi$ with $M=\mathcal{O}(k \log(N/k))$ measurements, and an $l_1$ minimization or a greedy matching pursuit decoder $\Delta$ promoting sparsity, with high probability CS leads to accurate signal reconstruction: The obtained approximation error is tightly upper bounded by a constant times the best $k$-term approximation error, the minimum error that one may achieve by keeping the $k$ largest coefficients in $\ba$~\cite{candes2006robust,candes2005decoding,cohen2009compressed,donoho2006compressed}. Redundant and signal adaptive dictionaries that further improve the CS performance with respect to orthogonal bases have been investigated~\cite{candes2010compressed, duarte2009learning, peyre2010best}. In addition to sparse models, manifold models have been considered for CS as well~\cite{baraniuk2009random,chen2010compressive}. 

The present paper introduces a novel framework of CS, namely \textit{statistical compressed sensing} (SCS). As opposed to conventional CS that deals
with one signal at a time, SCS aims at efficiently sampling a collection of signals and having accurate reconstruction on average. Instead of restricting to sparse models, SCS works with general Bayesian models. Assuming that the signals $\bx$ follow a distribution with probability density function (pdf) $f(\bx)$, SCS designs encoder-decoder pairs $(\sample, \Delta)$ so that the average error 
\vspace{0ex}
\begin{equation}
\label{eqn:mean:error:scs}
 \expect_{\bx} \|\bx - \Delta(\Phi \bx) \|_X = \int  \|\bx - \Delta(\Phi \bx) \|_X f(\bx) d\bx, \nonumber
\vspace{0ex}
\end{equation}
where $ \|\cdot\|_X$ is a norm, is small. As an important example, SCS with Gaussian models is here shown to have improved performance (bounds) relative to conventional CS, the signal reconstruction calculated with an optimal decoder $\Delta$ implemented via a fast linear filtering. Moreover, for Gaussian mixture models (GMMs) that better describe most real signals, SCS with a piecewise linear decoder is investigated.

The motivation of SCS with Gaussian models is twofold. First, controlling the average error over a collection of signals is 
useful in signal acquisition, not only because one is often interested in acquiring a collection of signals in real applications, but also because
 more effective processing of an individual signal, an image or a sound for example, is usually achieved by dividing the signal
 in (often overlapping) local subparts, patches (see Figure~\ref{fig:patches}) or short-time windows for instance, so a signal can be regarded as a collection  of subpart signals~\cite{aharon2006k,buades2006review, yu2008audio, yu2010PLE}. In addition, Gaussian mixture models (GMMs), which model
 signals or subpart signals with a collection of Gaussians, assuming each signal drawn from one of them, 
  have been shown effective in describing real signals,
leading to state-of-the-art results in image inverse problems~\cite{yu2010PLE} and missing data estimation~\cite{leger2010Matrix}. 

SCS based on a single Gaussian model is first developed in Section~\ref{sec:SCS:GM}. Following a similar 
mathematical approach as the one adopted in conventional CS performance analysis~\cite{cohen2009compressed}, it is shown that with the same random matrices as in conventional CS, but with a considerably reduced number $M=\mathcal{O}(k)$ of measurements, and with the optimal decoder implemented via linear filtering, significantly faster than the decoders applied in conventional CS, the average error of Gaussian SCS is tightly upper bounded by a constant times the best $k$-term approximation error with overwhelming probability, the failure probability being orders of magnitude smaller than that of conventional CS. 
Moreover, stronger yet simpler results further show that for \textit{any} sensing matrix, the average error of Gaussian SCS is upper bounded by a constant times the best $k$-term approximation with probability one, and the bound constant can be efficiently calculated. 

Section~\ref{sec:SCS:GMM:analysis} extends SCS to GMMs. A piecewise linear GMM-based SCS decoder, which essentially consists of estimating the signal using each Gaussian model included in the GMM and then selecting the best model, is introduced. The accuracy of the model selection, at the heart of the scheme, is analyzed in detail in terms of the properties of the Gaussian distributions and the number of sensing measurements. These results are then important in the general area of model selection from compressed measurements.

Following~\cite{yu2010PLE}, Section~\ref{sec:SCS:GMM:algorithm} presents an maximum a posteriori expectation-maximization (MAP-EM) algorithm that iteratively estimates the Gaussian models and decodes the signals. GMM-based SCS calculated with the MAP-EM algorithm is applied in real image sensing, leading to improved results with respect to conventional CS, at a considerably lower computational cost.

\section{Performance Bounds for a Single Gaussian Model}
\label{sec:SCS:GM}
This section analyzes the performance bounds of SCS based on a single Gaussian model. Perfect reconstruction 
of degenerated Gaussian signals is briefly discussed. After reviewing basic properties of linear approximation for Gaussian signals, the rest of the section shows that for Gaussian signals with fast eigenvalue decay, the average error of SCS using $k$ measurements and decoded by a linear estimator is tightly upper bounded by that of best $k$-term approximation. 

Signals $\bx \in \mathbb{R}^N$ are assumed to follow a Gaussian distribution $\mathcal{N}(\mu, \Sigma)$ in this section. Principal Component Analysis (PCA) calculates a basis change $\ba = \bB^T (\bx - \mu)$ of the data $\bx$, with $\bB$ the orthonormal PCA basis that diagonalizes the data covariance matrix 
\begin{equation}
\label{eqn:PCA}
\Sigma = \bB \bS\bB^T,
\end{equation} 
where $\bS = \mathrm{diag}(\lambda_{1}, \ldots, \lambda_{N})$ is a diagonal matrix whose diagonal elements $\lambda_1 \geq \lambda_2 \geq \ldots \geq \lambda_N$ are the sorted eigenvalues, and $\ba \sim \mathcal{N}(\bzero, \bS)$ the PCA coefficient vector~\cite{mallat2008wts}. In this section, for most of the time we will assume without loss of generality that $\bx \sim \mathcal{N}(\bzero, \bS)$ by looking in the PCA domain. For Gaussian and Bernoulli matrices that are known to be universal, analyzing CS in canonical basis or PCA basis is equivalent~\cite{baraniuk2008simple}. 

\subsection{Degenerated Gaussians}
\label{sec:degenerated:gaussian}

Conventional CS is able to perfectly reconstruct $k$-sparse signals, i.e., $\bx \in \mathbb{R}^N$ with at most $k$ non-zero entries (typically $k \ll N$), with $2k$ measurements~\cite{cohen2009compressed}. Degenerated Gaussian distributions $\mathcal{N}(\bzero, \bS_k)$, where $\bS_k = \mathrm{diag}(\lambda_{1}, \ldots, \lambda_{k}, 0, \ldots, 0)$ with at most $k$ non-zero eigenvalues, give the counterpart of $k$-sparsity for the Gaussian signal models considered in this paper. Such signals belong to a linear subspace $\mathcal{S}_k = \{\bx| \bx[m] = 0, \forall k < m \leq N\}$. 
The next lemma gives a condition for perfect reconstruction of signals in $\mathcal{S}_k$. 

\begin{lemma}
If $\Phi$ is any $M \times N$ matrix and $k$ is a positive integer, then there is a decoder $\Delta$ such that $\Delta(\Phi \bx) = \bx$, for all $\bx \in \mathcal{S}_k$, if and only if $\mathcal{S}_k \cap \nsp(\Phi) = \bzero$. 
\end{lemma}
\begin{proof}
Suppose there is a decoder $\Delta$ such that $\Delta(\Phi \bx) = \bx$, for all $\bx \in \mathcal{S}_k$. Let $\bx = \mathcal{S}_k \cap \nsp(\Phi)$. We can write $\bx  = \bx_1 - \bx_2$ where both $\bx_1, \bx_2 \in \mathcal{S}_k$. Since $\Phi \bx = \bzero$, $\Phi \bx_1 = \Phi \bx_2$. Plugging $\Phi \bx_1$ and $\Phi \bx_2$ into the decoder $\Delta$, we obtain $\bx_1 = \bx_2$ and then $\bx  = \bx_1 - \bx_2 = \bzero$.

Suppose $\mathcal{S}_k \cap \nsp(\Phi) = \bzero$. If $\bx_1, \bx_2 \in \mathcal{S}_k$ with $\Phi \bx_1 = \Phi \bx_2$, then $\bx_1 - \bx_2 \in \mathcal{S}_k \cap \nsp(\Phi)$, so $\bx_1 = \bx_2$. $\Phi$ is thus a \textit{one-to-one map}. Therefore there must exist a decoder $\Delta$ such that $\Delta(\Phi \bx) = \bx$.
\end{proof}

It is possible to construct matrices of size $M \times N$ with $M=k$ which satisfies the requirement of the Lemma. A trivial example is $[\bI_{M \times M} |\bzero_{M \times (N-M)}]$, where $\bI_{M \times M}$ is the identity matrix of size $M \times M$ and $\bzero_{M \times (N-M)}$ is a zero matrix of size $M \times (N-M)$. Comparing with conventional compressed sensing that requires $2k$ measurements for exact reconstruction of $k$-sparse signals, with only $k$ measurements signals in $\mathcal{S}_k$ can be exactly reconstructed. Indeed, in contrast to the $k$-sparse signals where the positions of the non-zero entries are unknown ($k$-sparse signals reside in multiple $k$-dimensional subspaces), with the degenerated Gaussian model $\mathcal{N}(\bzero, \bS_k)$, the position of the non-zero coefficients are known a priori to be the first $k$ ones. $k$ measurements thus suffice for perfect reconstruction.   

In the following, we will concentrate on the more general case of non-degenerated Gaussian signals (i.e., Gaussians with full-rank covariance matrices $\Sigma$) with fast eigenvalue decay, in analogy to \textit{compressible} signals for conventional CS. As mentioned before and will be further experimented in this paper, such simple models not only lead to improved theoretical bounds, but also provide state-of-the-art image reconstruction results.

\subsection{Optimal Decoder}
\label{subsec:optimal:decoder}

To simplify the notation, we assume without loss of generality that the Gaussian has zero mean ${\mu} = \mathbf{0}$, as one can always center the signal with respect to the mean. 

It is well-known that the optimal decoders for Gaussian signals are calculated with linear filtering:
\vspace{0ex}
\begin{theorem}\hspace{-1ex}~\cite{kay1998fundamentals}
\label{theo:gaussian:MAP}
Let $\bx \in \mathbb{R}^N$ be a random vector with prior pdf $\mathcal{N}(\bzero, \Sigma)$, and $\Phi \in \mathbb{R}^{M \times N}$, $M \leq N$, be a sensing matrix. From the measured signal $\by = \sample \bx \in \mathbb{R}^M$, the optimal decoder $\Delta$ that minimizes the mean square error (MSE)
$
 \expect_{\bx}  [\|\bx - \Delta (\Phi \bx)\|_2^2]  =  \min_{g} \expect_{\bx} [\|\bx - g(\Phi \bx)\|_2^2], \nonumber 
$
as well as the mean absolute error (MAE)
$
 \expect_{\bx}  [\|\bx - \Delta (\Phi \bx)\|_1]  =  \min_{g} \expect_{\bx} [\|\bx -  g(\Phi \bx)\|_1], \nonumber 
$
where $g: \mathbb{R}^{M}  \rightarrow \mathbb{R}^{N}$, is obtained with a linear MAP estimator,
\vspace{0ex}
\begin{equation}
\label{eqn:MAP:Sigma} 
\Delta (\Phi \bx) = \arg\max_{\bx} p(\bx|\by) = \underbrace{\Sigma \Phi^T (\Phi \Sigma \Phi^T)^{-1}}_{\Delta} (\Phi \bx),
\vspace{0ex}
\end{equation}
and the resulting error $\eta = \bx - \Delta (\Phi \bx)$ is Gaussian with mean zero and with covariance matrix 
$
\Sigma_\eta  =  \expect_{\bx}  [\eta \eta^T]  
 =  \Sigma - \Sigma \Phi^T (\Phi \Sigma \Phi^T)^{-1} \Phi  \Sigma,
$
whose trace yields the MSE of SCS. \vspace{0ex}
\begin{equation}
\label{eqn:MSE:GSCS:Sigma}
 \expect_{\bx}  [\|\bx - \Delta (\Phi \bx)\|_2^2]  = Tr (\Sigma - \Sigma \Phi^T (\Phi \Sigma \Phi^T)^{-1} \Phi  \Sigma). 
 \vspace{0ex}
 \end{equation}
\end{theorem}
In contrast to conventional CS, for which the $l_1$ minimization or greedy matching pursuit decoders, calculated with iterative procedures, have been shown optimal under certain conditions on $\Phi$ and the signal sparsity~\cite{candes2006robust, donoho2006compressed}, Gaussian SCS enjoys the advantage of having an optimal decoder~\eqref{eqn:MAP:Sigma}  calculated fast via a closed-form \textit{linear} filtering for \textit{any} $\Phi$.

\begin{corollary}
\label{coro:error:MAP:Phi:random}
If a random matrix $\Phi \in \mathbb{R}^{M \times N}$ is drawn independently to sense each $\bx$, with all the other conditions as in Theorem~\ref{theo:gaussian:MAP}, the MSE of SCS is
\vspace{0ex}
\begin{equation}
\label{eqn:MSE:GSCS:randPhi:Sigma}
 \expect_{\bx, \Phi} [\|\bx - \Delta(\Phi \bx) \|_2^2] =  \expect_{\Phi} 
[Tr (\Sigma - \Sigma \Phi^T (\Phi \Sigma \Phi^T)^{-1} \Phi \Sigma)].\vspace{0ex}
 \end{equation}
\end{corollary}
Applying an independent random matrix realization to sense each signal has been considered in~\cite{cohen2009compressed}. In real applications, these random sensing matrices need \textit{not} to be stored, since the decoder can regenerate them itself given the random seed. 

Following a PCA basis change~\eqref{eqn:PCA}, it is equivalent to consider signals $\bx \sim \mathcal{N}(\mu, \Sigma)$ and $\bx \sim \mathcal{N}(\bzero, \bS)$, where $\bS = \mathrm{diag}(\lambda_{1}, \ldots, \lambda_{N})$ is a diagonal matrix whose diagonal elements $\lambda_1 \geq \lambda_2 \geq \ldots \geq \lambda_N$ are the sorted eigenvalues. Theorem~\ref{theo:gaussian:MAP} and Corollary~\ref{coro:error:MAP:Phi:random} clearly hold for $\bx \sim \mathcal{N}(\bzero, \bS)$, with~\eqref{eqn:MAP:Sigma},~\eqref{eqn:MSE:GSCS:Sigma}, and~\eqref{eqn:MSE:GSCS:randPhi:Sigma} rewritten as 
\begin{eqnarray}
\label{eqn:MAP} 
\Delta (\Phi \bx)  & = & \underbrace{\bS \Phi^T (\Phi \bS \Phi^T)^{-1}}_{\Delta} (\Phi \bx), \\
\label{eqn:MSE:GSCS}
 \expect_{\bx}  [\|\bx - \Delta (\Phi \bx)\|_2^2]  & = &Tr (\bS - \bS \Phi^T (\Phi \bS \Phi^T)^{-1} \Phi  \bS), \\
 \label{eqn:MSE:GSCS:randPhi}
 \expect_{\bx, \Phi} [\|\bx - \Delta(\Phi \bx) \|_2^2] & = & \expect_{\Phi} 
[Tr (\bS - \bS \Phi^T (\Phi \bS \Phi^T)^{-1} \Phi \bS)].
\end{eqnarray}

Note that PCA bases, as sparsifying dictionaries, have been applied to do conventional CS based on sparse models~\cite{masiero2010data},  which is fundamentally different than the Gaussian models and SCS here studied.

\subsection{Linear vs Nonlinear Approximation}
\label{subsec:linear:vs:nonlinear}

Before proceeding with the analysis of the SCS performance, let us make some comments on the relationship between linear and non-linear approximations for Gaussian signals. In particular, the following is observed via Monte Carlo simulations:

{\it For Gaussian signals  $\bx \sim \mathcal{N}(\bzero, \bS)$, where $\bS = \mathrm{diag}(\lambda_{1}, \ldots, \lambda_{N})$ whose eigenvalues $\lambda_1 \geq \lambda_2 \geq \ldots \geq \lambda_N$  decay fast, the best $k$-term  linear approximation 
\begin{equation}
\label{eqn:linear:approx} 
{\bx}^l_k(m) =  
\left\{ 
\begin{array}{cc}
\bx(m) & ~~1\leq m \leq k, \\
0 &~~ k+1 \leq m \leq N,
\end{array}
\right.
\end{equation}
and the nonlinear approximation 
\begin{equation}
{\bx}^n_k = T_k (\bx),
\end{equation}
where $T_k$ is a thresholding operator that keeps the $k$ coefficients of largest amplitude and setting others to zero, lead to comparable approximation errors 
\begin{equation}
\label{eqn:err:best:k}
\sigma^l_{k}(\{\bx\})_X = \expect_{\bx} [\|\bx - {\bx}^l_k \|_X]~~~\textrm{and}~~~\sigma^n_{k}(\{\bx\})_X = \expect_{\bx} [\|\bx - {\bx}^n_k \|_X].
\end{equation}}

Monte Carlo simulations are performed to test this.  Assuming a power decay of the eigenvalues~\cite{mallat2008wts},
\vspace{-1ex}
\begin{equation}
\label{eqn:eigvalue:power:decay}
\lambda_m = m^{-\alpha},~~~1 \leq m \leq N,
\vspace{-1.5ex}
\end{equation}
where $\alpha > 0$ is the decay parameter, with $N=64$, Figure~\ref{fig:approx:err:linear:nonlinear} plots the MSEs 
\begin{equation}
\label{eqn:MSE:best:k}
\sigma^l_{k}(\{\bx\})_2^2 = \expect_{\bx} [\|\bx - {\bx}^l_k \|_2^2]~~~\textrm{and}~~~\sigma^n_{k}(\{\bx\})_2^2 = \expect_{\bx} [\|\bx - {\bx}^n_k \|_2^2],
\end{equation}
normalized by the ideal signal energy $\|\bx\|_2^2$, of best $k$-term linear and nonlinear approximations as a function of $\alpha$, with typical (for image patches of size $8 \times 8$ for example) $k$ values $8$ and $16$ ($k/N=1/8$ and $1/4$). Both MSEs decrease as $\alpha$ increases, i.e., as the eigenvalues decay faster. With typical values $\alpha \approx 3$ (similar to the eigenvalue decay calculated with typical image patches) and $k = 8$ or 16, both approximations are accurate and generate small and comparable MSEs, their difference being about $0.1\%$ of the signal energy and ratio about 2. 

\begin{figure}[htbp]
\vspace{-13ex}
\begin{center}
\begin{tabular}{ccc}
\hspace{0ex} \includegraphics[width=6cm]{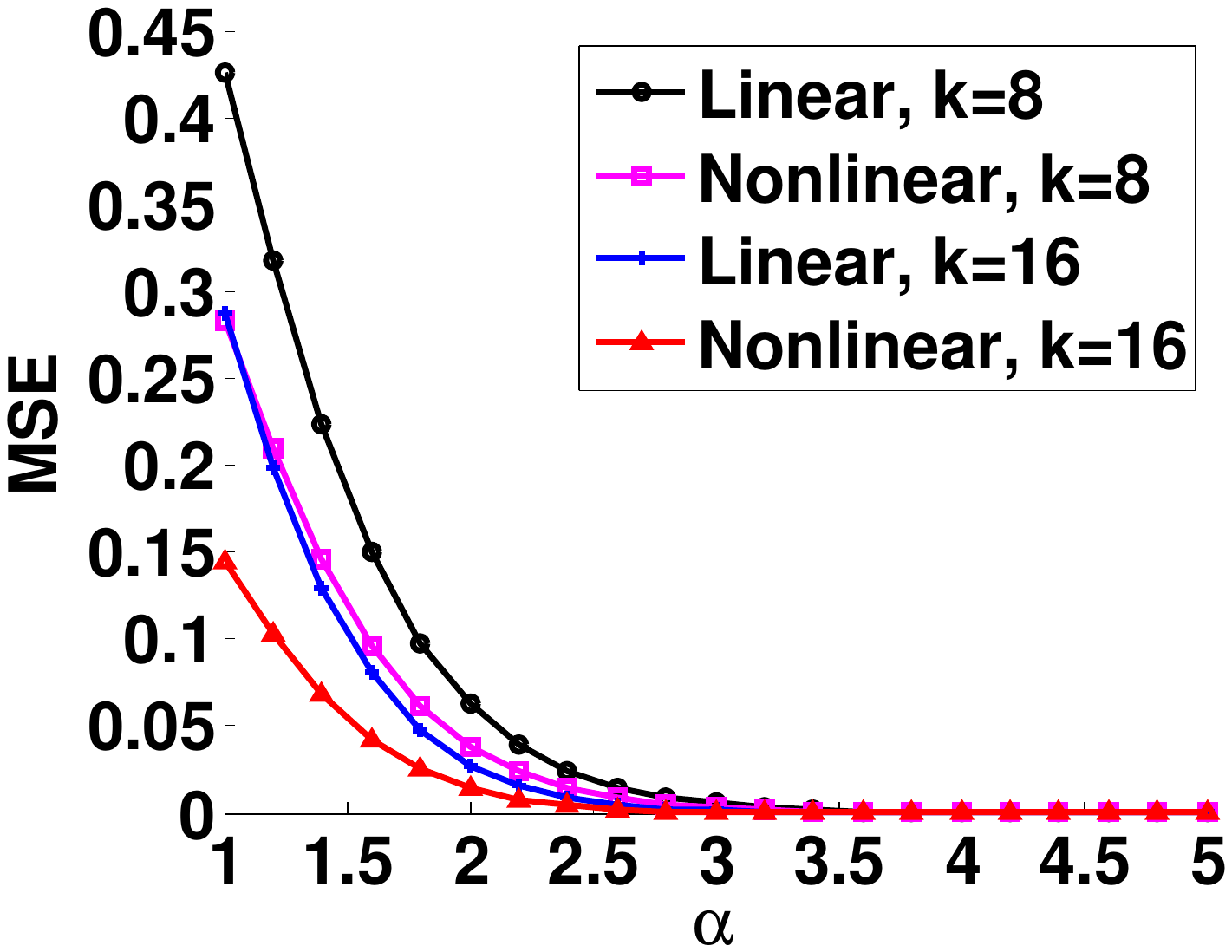} &
\hspace{-10ex} \includegraphics[width=6cm]{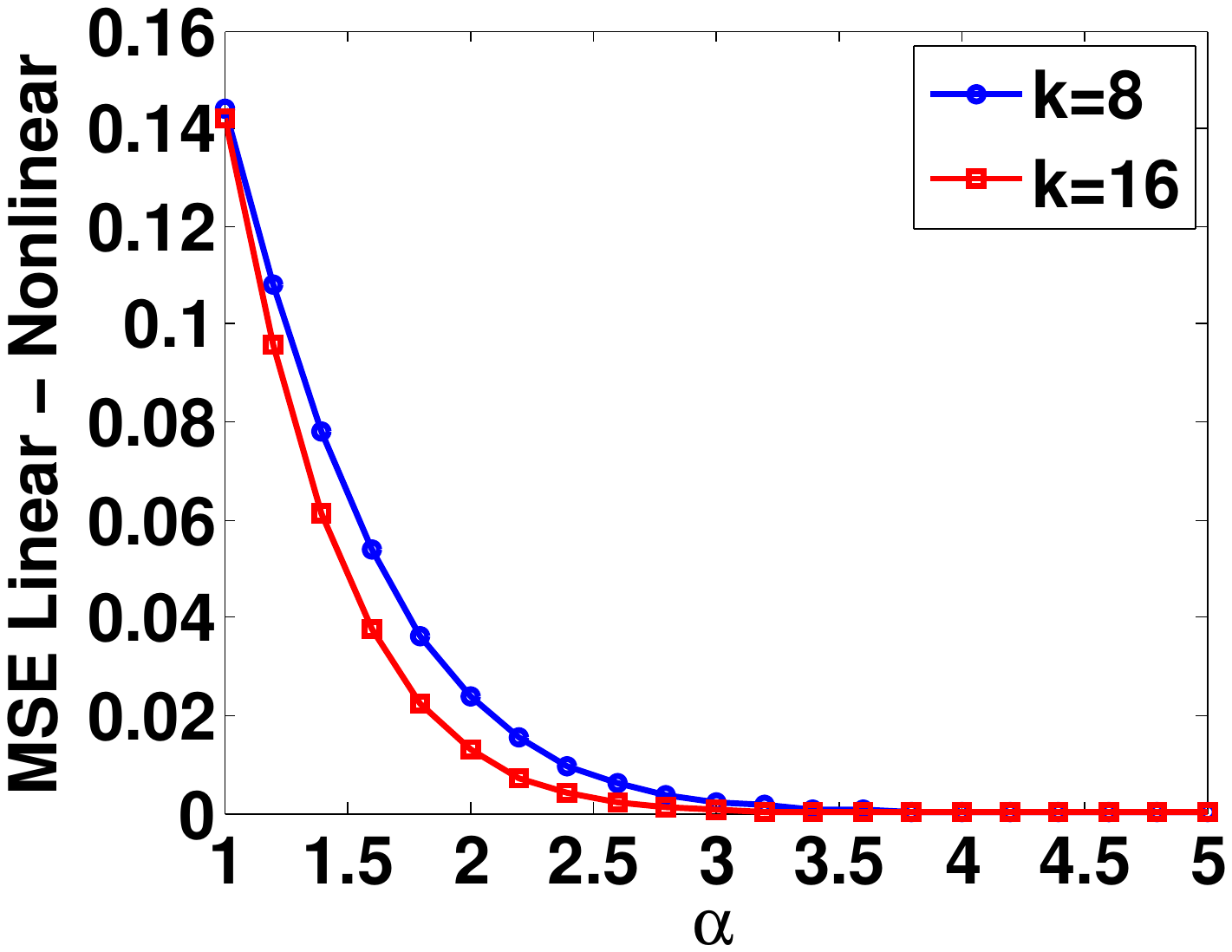} &
\hspace{-10ex} \includegraphics[width=6cm]{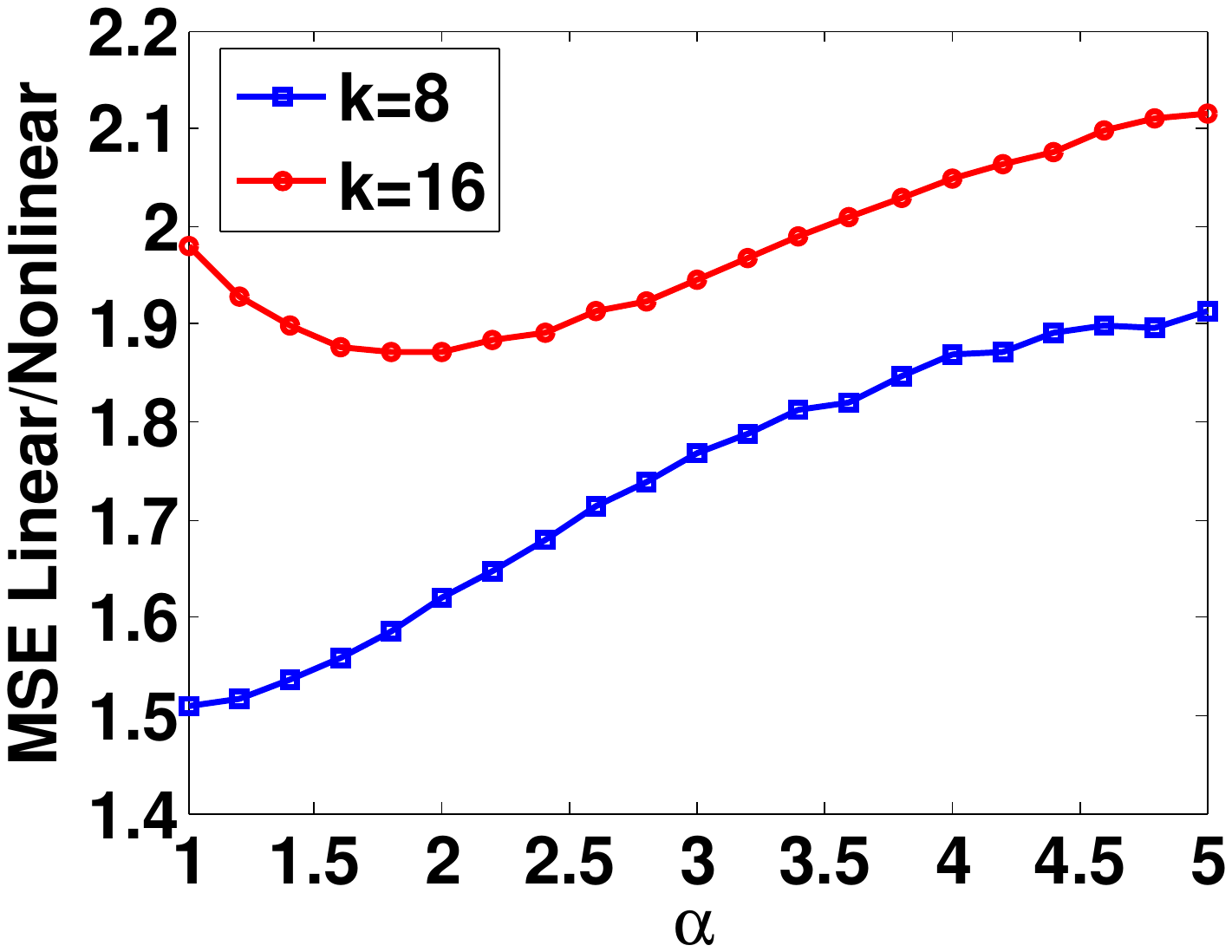} \vspace{-13ex} \\
\hspace{0ex} \textbf{(a)} &
\hspace{-10ex}  \textbf{(b)}&
\hspace{-10ex}  \textbf{(c)}\\
\end{tabular}
\end{center}
\caption{(a). MSEs (normalized by the ideal signal energy) of best $k$-term  linear and non-linear approximation, with $k=8$ and $16$ (signal dimension $N=64$). (b) and (c) Difference and ratio of normalized MSEs of  best $k$-term linear and non-linear approximation shown in (a).} \label{fig:approx:err:linear:nonlinear}
\vspace{0ex}
\end{figure}

Following this, the error of Gaussian SCS will be compared with that of best $k$-term linear approximation, which is comparable to that of best $k$-term nonlinear approximation. For simplicity, the best $k$-term linear approximation errors will be denoted as 
\begin{equation}
\label{eqn:err:best:k:linear}
\sigma_{k}(\{\bx\})_X = \sigma^l_{k}(\{\bx\})_X~~~\textrm{and}~~~\sigma_{k}(\{\bx\})_2^2 = \sigma^l_{k}(\{\bx\})_2^2.
\end{equation}
Note that $\sigma_{k}(\{\bx\})_2^2 = \sum_{m=k+1}^N \lambda_m$. 

\subsection{Performance of Gaussian SCS -- A Numerical Analysis At First}
\label{subsec:SCS:numeric:analysis}

This section numerically evaluates the MSE of Gaussian SCS, and compares it with the minimal MSE generated by best $k$-term linear approximation, proceeding the theoretical bounds later developed.

As before, a power decay of the eigenvalues~\eqref{eqn:eigvalue:power:decay}, with $N=64$, is assumed in the Monte Carlo simulations. An independent random Gaussian matrix realization $\Phi$ is applied to sense each signal $\bx$~\cite{cohen2009compressed}. 

Figures~\ref{fig:MSE:scs:vs:bestk} (a) and (c)-top plot the MSE (normalized by the ideal signal energy) of SCS and that of the best $k$-term linear approximation, as well as their ratio as a function of $\alpha$, with $k$ fixed at typical values $8$ and $16$ ($k/N=1/8$ and $1/4$). As $\alpha$ increases, i.e., as the eigenvalues decay faster, the MSEs for both methods decrease. Their ratio increases almost linearly with $\alpha$. The same is plotted in figures~\ref{fig:MSE:scs:vs:bestk} (b) and (c)-bottom, with eigenvalue decay parameter fixed at a typical value $\alpha=3$, and with $k$ varying from $5$ to $32$ ($k/N$ from $5/64$ to $1/2$). As $k$ increases, both MSEs decrease, their ratio being almost constant at about $3.7$. 
\begin{figure}[htbp]
\vspace{-13ex}
\begin{center}
\begin{tabular}{ccc}
\hspace{0ex} \includegraphics[width=6cm]{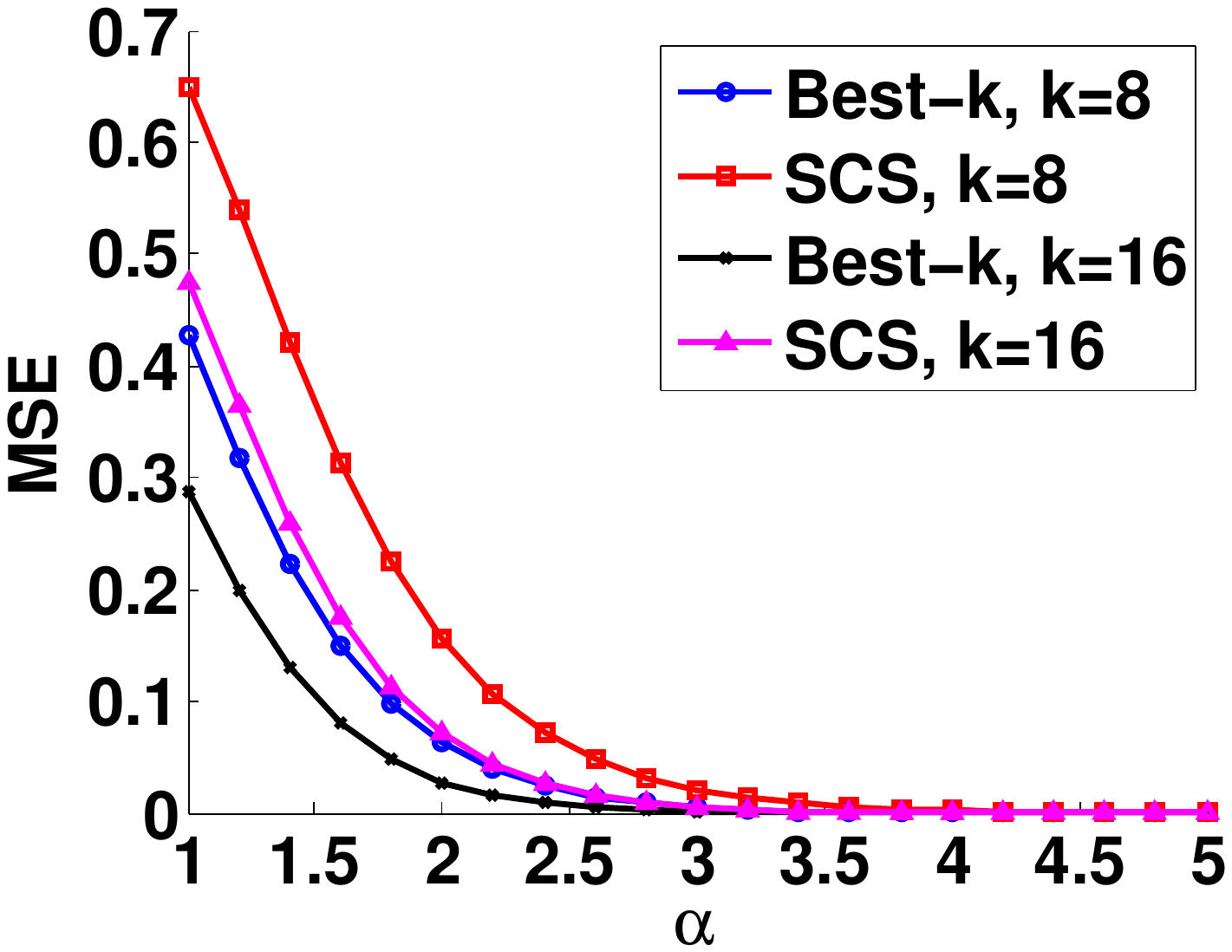} &
\hspace{-10ex} \includegraphics[width=6cm]{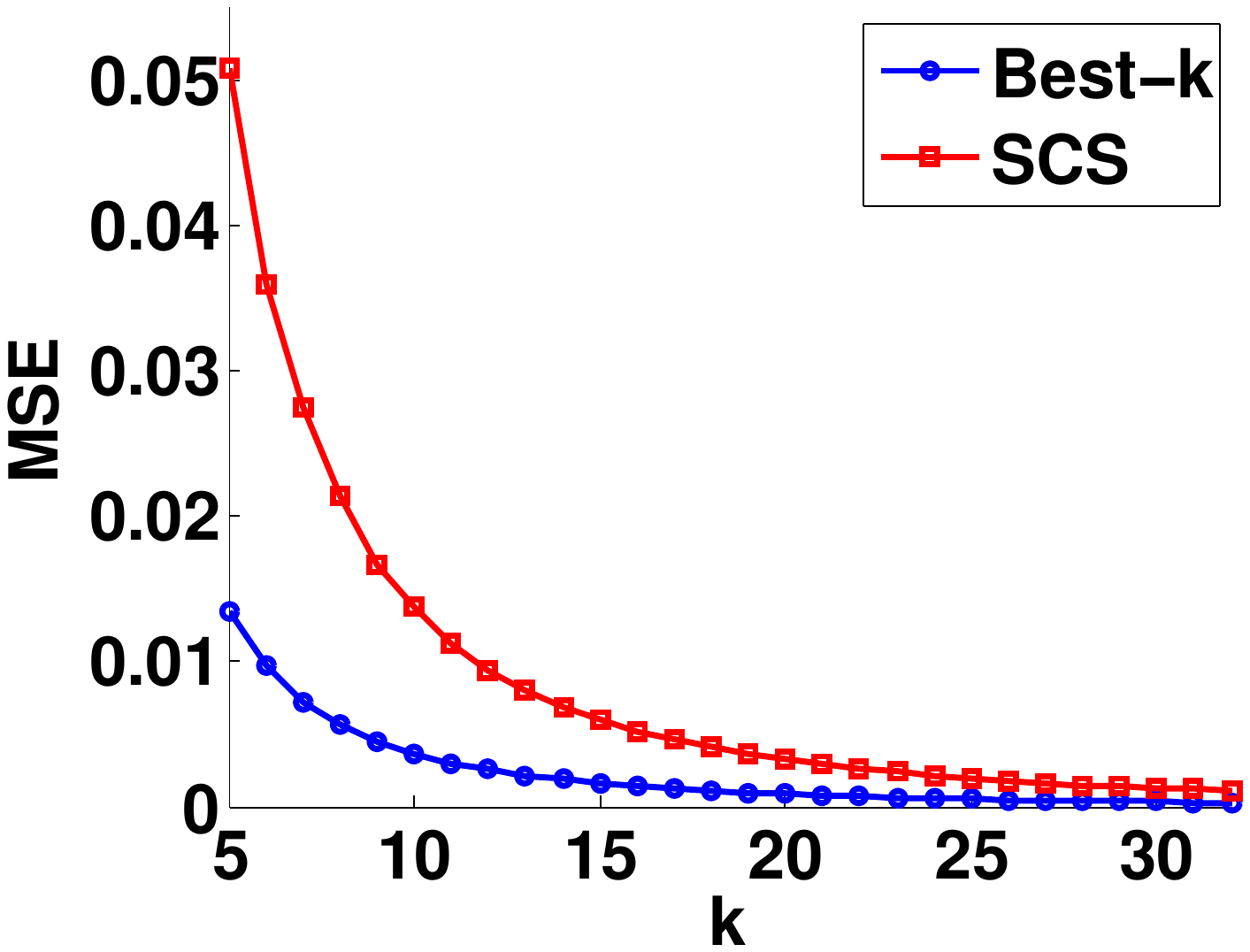} &
\hspace{-10ex} \includegraphics[width=6cm]{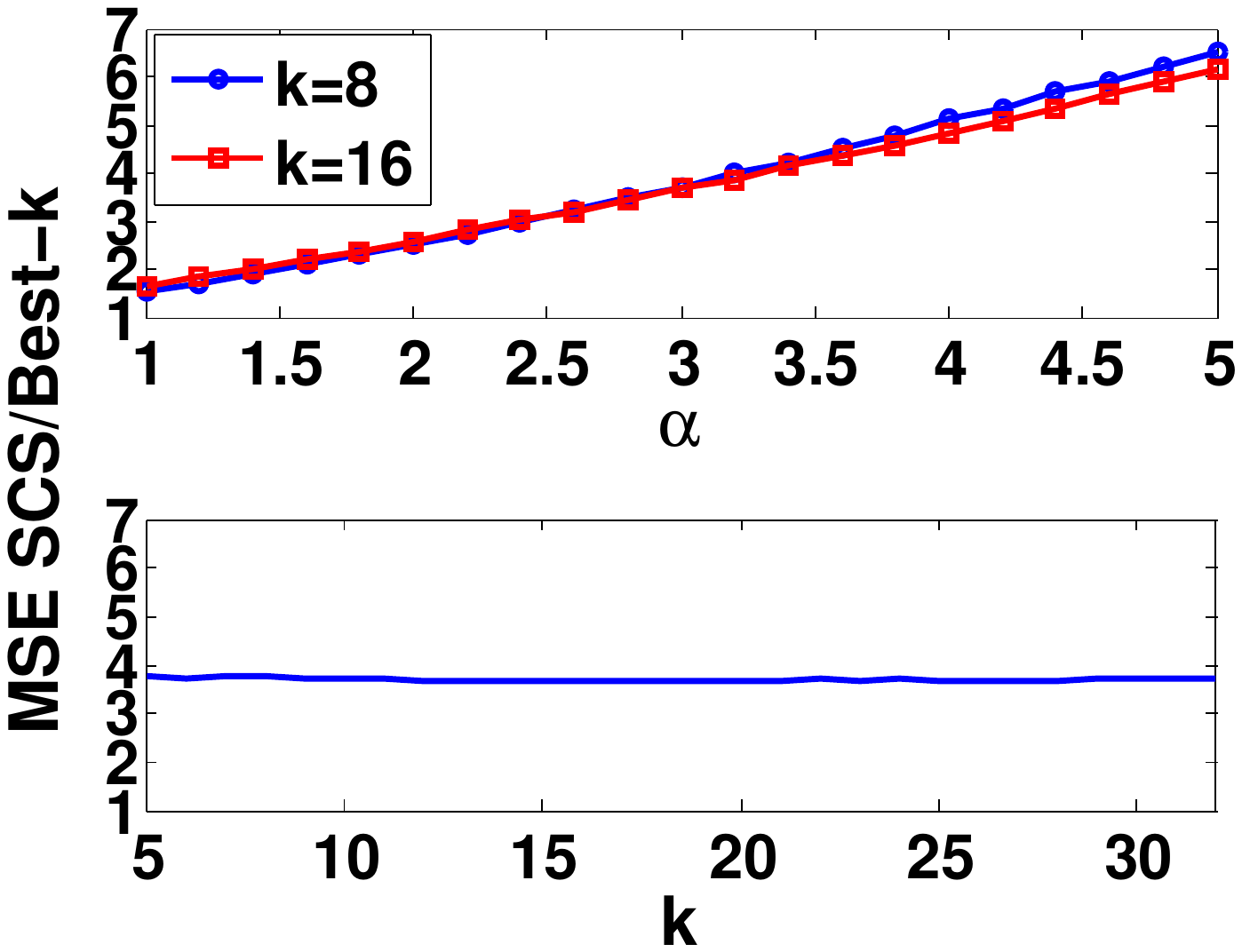} \vspace{-13ex} \\
\hspace{0ex} \textbf{(a)} &
\hspace{-10ex}  \textbf{(b)}&
\hspace{-10ex}  \textbf{(c)}\\
\end{tabular}
\end{center}
\vspace{0ex}
\caption{Comparison of the MSE of SCS and that of the best $k$-term linear approximation for Gaussian signals of dimension $N=64$. (a) and (c)-top. The MSE (normalized by the ideal signal energy) of SCS and that of best $k$-term linear approximation, as well as their ratio as a function of $\alpha$, with $k$ fixed at typical values $8$ and $16$. (b) and (c)-bottom. The same values, with eigenvalue decay parameter fixed at a typical value $\alpha=3$, and with $k$ varying from $5$ to $32$.} \label{fig:MSE:scs:vs:bestk}
\vspace{0ex}
\end{figure}

These results indicate a good performance of Gaussian SCS, its MSE is only a small number of times larger than that of the best $k$-term linear approximation.~\footnote{Simulations using the same coefficient energy power decay
model~\eqref{eqn:eigvalue:power:decay} show that the ratio between
conventional CS based on sparse models, with $k$ measurements, and that of
the best $k$-term nonlinear approximation, varies as a function of the
decay parameter $\alpha$ and $k$. For typical values $\alpha=3$, the ratio
is typically an order of magnitude larger than that between the MSE of SCS
and that of the best $k$-term linear approximation.} The next sections provide mathematical analysis of this performance. 

Let us notice that while the best $k$-term linear approximation decoding is feasible for signals following a single Gaussian distribution, it is impractical with GMMs (assuming multiple Gaussians and that each signal is generated from one of them with an unknown index), since the Gaussian index of the signal is unknown. SCS with GMMs, which describe real data considerably better than a single Gaussian model~\cite{yu2010PLE}, will be described in sections~\ref{sec:SCS:GMM:analysis} and~\ref{sec:SCS:GMM:algorithm}.

\subsection{Performance Bounds}

Following the analysis techniques in~\cite{cohen2009compressed}, this section shows that with Gaussian and Bernoulli random matrices of $\mathcal{O}(k)$ measurements, considerably smaller than the $\mathcal{O}(k \log (N/k))$ required by conventional CS, the average error of Gaussian SCS is tightly upper bounded by a constant times the best $k$-term linear approximation error with overwhelming probability, the failure probability being orders of magnitude smaller than that of conventional CS. 

We consider only the encoder-decoder pairs $(\Phi, \Delta)$ that preserve $\Phi \bx$, i.e., $\Phi (\Delta (\Phi \bx)) = \Phi \bx$, satisfied by the optimal $\Delta$ in~\eqref{eqn:MAP} for Gaussian signals $\bx$,  $\forall \Phi$. 

\subsubsection{From Null Space Property to Instance Optimality}
The \textit{instance optimality in expectation} bounds the average error of SCS with a constant times that of the best $k$-term linear approximation~\eqref{eqn:err:best:k:linear}, defining the desired SCS performance:
\vspace{0ex}
\begin{definition} 
Let $\bx \in \mathbb{R}^N$ be a random vector that follows a certain distribution. Let $K\subset \{1, \ldots, N\}$ be any subset of indices. We say that $(\Phi, \Delta)$ is \textit{instance optimal in expectation} in $K$ in $\|\cdot\|_X$, with a constant $C_0$, if 
\vspace{0ex}
\begin{equation}
\label{eqn:instance:optimality}
E_{\bx, (\Phi)} [\|\bx - \Delta (\Phi \bx)\|_X] \leq C_0 E_{\bx} [\|\bx - \bx_K\|_X],
\vspace{0ex}
\end{equation}
where $\bx_K$ is the signal $\bx$ restricted to $K$ ($\bx_K[n]=\bx[n],~\forall~n \in K$, and $0$ otherwise),
the expectation on the left side considered with respect to $\bx$, and to $\Phi$ if one random $\Phi$ is drawn independently for each $\bx$. Similarly, the MSE instance optimality in $K$ is defined as 
\vspace{0ex}
\begin{equation}
\label{eqn:instance:optimality:MSE}
E_{\bx, (\Phi)} [\|\bx - \Delta (\Phi \bx)\|_2^2] \leq C_0 E_{\bx} [\|\bx - \bx_K\|_2^2].
\vspace{0ex}
\end{equation}

In particular, if $K = \{1, \ldots, k\}$,  then we say that $(\Phi, \Delta)$ is \textit{instance optimal in expectation} of order $k$ in $\|\cdot\|_X$, with a constant $C_0$, if 
\begin{equation}
\label{eqn:instance:optimality:firstk}
E_{\bx, (\Phi)} [\|\bx - \Delta (\Phi \bx)\|_X]  \leq C_0 E_{\bx} [\|\bx - \bx_K\|_X] = C_0 \sigma_{k}(\{\bx\})_X,
\vspace{0ex}
\end{equation}
and is instance optimal of order $k$ in MSE, with a constant $C_0$, if 
\begin{equation}
\label{eqn:instance:optimality:MSE:firstk}
E_{\bx, (\Phi)} [\|\bx - \Delta (\Phi \bx)\|_2^2] \leq C_0 E_{\bx} [\|\bx - \bx_K\|_2^2] = C_0 \sigma_{k}(\{\bx\})_2^2.
\vspace{0ex}
\end{equation}
\end{definition}

The \textit{null space property in expectation} defined next will be shown equivalent to the instance optimality in expectation.
\vspace{0ex}
\begin{definition} 
Let $\bx \in \mathbb{R}^N$ be a random vector that follows a certain distribution. Let $K\subset \{1, \ldots, N\}$ be any subset of indices. We say that $\Phi$ in $(\Phi, \Delta)$ has the null space property  in expectation in $K$ in $\|\cdot\|_X$, with constant $C$, if 
\vspace{0ex}
\begin{equation}
\label{eqn:null:space:property}
E_{\bx, (\Phi)} [\|\eta\|_X] \leq C E_\bx [\|\eta - \eta_K\|_X],~\textrm{where}~\eta = \bx -  \Delta (\Phi \bx),
\end{equation}
where $\eta_K$ is the signal $\eta$ restricted to $K$ ($\eta_K[n]=\eta[n],~\forall~n \in K$, and $0$ otherwise),
the expectation considered on the left side with respect to $\bx$, and to $\Phi$ if one random $\Phi$ is drawn independently for each $\bx$. Note that $\eta \in \nsp(\Phi)$. Similarly, the MSE null space property in $K$ is defined as 
\begin{equation}
\label{eqn:null:space:property:MSE}
E_{\bx, (\Phi)} \|\eta\|_2^2 \leq C E_\bx [\|\eta - \eta_K\|_2^2],~\textrm{where}~\eta = \bx -  \Delta (\Phi \bx).
\end{equation}

In particular, if $K = \{1, \ldots, k\}$, with $1 \leq k \leq N$, then we say that $\Phi$ in $(\Phi, \Delta)$ has the null space property in expectation of order $k$ in $\|\cdot\|_X$, with constant $C$, if 
\vspace{0ex}
\begin{equation}
\label{eqn:null:space:property:firstk}
E_{\bx, (\Phi)} \|\eta\|_X \leq C E_\bx [\|\eta - \eta_K\|_X] = C \sigma_{k}(\{\eta\})_X,~\textrm{where}~\eta = \bx -  \Delta (\Phi \bx),
\end{equation}
and has the MSE null space property of order $k$, if 
\begin{equation}
\label{eqn:null:space:property:MSE:firstk}
E_{\bx, (\Phi)} \|\eta\|_2^2 \leq C E_\bx [\|\eta - \eta_K\|_2^2] = C \sigma_{k}(\{\eta\})_2^2,~\textrm{where}~\eta = \bx -  \Delta (\Phi \bx).
\end{equation}

\end{definition}

\begin{theorem}
\label{theo:nullspace2optimality}
Let $\bx \in \mathbb{R}^N$ be a random vector that follows a certain distribution. Given an $M \times N$ matrix $\Phi$, a norm $\|\cdot\|_X$, and a subset  of indices $K \subset \{1, \ldots, N\}$, a sufficient condition that there exists a decoder $\Delta$ such that the instance optimality in expectation in $K$ in $\|\cdot\|_X$~\eqref{eqn:instance:optimality} holds with constant $C_0$, is that the null space property in expectation~\eqref{eqn:null:space:property} holds with $C=C_0/2$ for this $(\Phi, \Delta)$:
\begin{equation}
\label{eqn:instance:optimality:sufficient}
E_{\bx, (\Phi)} [\|\eta\|_X] \leq  \frac{C_0}{2} E_\bx [\|\eta - \eta_K\|_X],~\textrm{where}~\eta = \bx -  \Delta (\Phi \bx).
\end{equation}
A necessary condition is the null space property in expectation~\eqref{eqn:null:space:property} with $C = C_0$:
\begin{equation}
\label{eqn:instance:optimality:necessary}
E_{\bx, (\Phi)} [\|\eta\|_X] \leq C_0 E_\bx [\|\eta - \eta_K\|_X],~\textrm{where}~\eta = \bx -  \Delta (\Phi \bx),
\end{equation}
Similar results hold between the MSE instance optimality in $K$~\eqref{eqn:instance:optimality:MSE} and the null space property~\eqref{eqn:null:space:property:MSE}, with the constant $C=C_0/4$ in the sufficient condition.

In particular, if $K = \{1, \ldots, k\}$, with $1 \leq k \leq N$, the same equivalence between the instance optimality in expectation of order $k$ in $\|\cdot\|_X$~\eqref{eqn:instance:optimality:firstk} and the null space property in expectation~\eqref{eqn:null:space:property:firstk}, and that between the MSE instance optimality of order $k$~\eqref{eqn:instance:optimality:MSE:firstk} and the null space property~\eqref{eqn:null:space:property:MSE:firstk},  hold as well. 
\end{theorem}

\begin{proof}
To prove the sufficiency of~\eqref{eqn:instance:optimality:sufficient}, we consider the decoder $\Delta$ such that for all $\by = \Phi \bx \in \mathbb{R}^M$, 
\begin{equation}
\label{eqn:choice:decoder}
\Delta(\by) := \arg \min_\bz \|\bz - \bz_K\|_X~~~\textrm{s.t.}~~~\Phi \bz = \by.
\end{equation}
By~\eqref{eqn:instance:optimality:sufficient}, we have
\begin{eqnarray}
E \|\bx -  \Delta (\Phi \bx)\|_X  
& \leq & \frac{C_0}{2} E_\bx [\|(\bx -  \Delta (\Phi \bx)) - (\bx_K -  (\Delta \Phi \bx)_K)\|_X] \\
& \leq & \frac{C_0}{2} (E_\bx [\|\bx -  \bx_K\|_X] + E_\bx [\|\Delta (\Phi \bx)  -  (\Delta \Phi \bx)_K\|_X] )\\
& \leq & {C_0} E_\bx [\|\bx -  \bx_K\|_X],  \nonumber 
\end{eqnarray}
where the second inequality uses the triangle inequality, and the last inequality follows from the choice of the decoder~\eqref{eqn:choice:decoder}. 

To prove the necessity of~\eqref{eqn:instance:optimality:necessary}, let $\Delta$ be any decoder for which~\eqref{eqn:instance:optimality}  holds. Let $\eta = \bx -  \Delta (\Phi \bx)$ and let $\eta_K$ be the linear approximation of $\eta$ in $K$ ($\eta_K[n]=\eta[n],~\forall~n \in K$, and $0$ otherwise). Let $\eta_K = \eta_1 + \eta_2$ be any splitting of $\eta_K$ into two vectors in the linear space $\mathcal{S}_K = \{\bx| \bx[m] = 0, \forall k \notin K\}$. We can write
$$
\eta = \eta_1 + \eta_2 + \eta_3,
$$
with $\eta_3 = \eta - \eta_K$. As the right side of~\eqref{eqn:instance:optimality} is equal to 0 for $\forall~\bx \in \mathcal{S}_K$, we deduce $-\eta_1 = \Delta(\Phi(- \eta_1))$. Since $\eta \in \nsp(\Phi)$, we have $\Phi(-\eta_1) = \Phi(\eta_2 +  \eta_3)$, so that $-\eta_1 = \Delta (\Phi(\eta_2 + \eta_3))$. We derive
\begin{eqnarray}
E [\|\eta\|_X] 
& = & E [\|\eta_2 + \eta_3 -  \Delta \Phi(\eta_2 +  \eta_3)\|_X] \leq C_0 E_\bx [\|(\eta_2 +  \eta_3) - ((\eta_2)_K +  (\eta_3)_K) \|_X] \nonumber\\
& = & C_0 E_\bx [\|\eta - \eta_K \|_X], \nonumber
\end{eqnarray}
where the inequality follows from~\eqref{eqn:instance:optimality}, and the second and third equalities use the fact that $\eta = \eta_1 + \eta_2 + \eta_3$ and $\eta_1 \in \mathcal{S}_K$. Thus we have obtained~\eqref{eqn:instance:optimality:necessary}. 

A similar proof proceeds for MSE instance optimality and null space property. The second part of the theorem is a direct consequence of the first part. 
\end{proof}

Comparing to conventional CS that requires the null space property to hold with the best $2k$-term nonlinear approximation error~\cite{cohen2009compressed}, the requirement for Gaussian SCS is relaxed to $k$, thanks to the linearity of the best $k$-term linear approximation for Gaussian signals. 

Theorem~\ref{theo:nullspace2optimality} proves the existence of the decoder $\Delta$ for which the instance optimality in expectation holds for $(\Phi, \Delta)$, given the null space property in expectation. However, it does not explain how such decoder is implemented. The following Corollary, a direct consequence of theorems~\ref{theo:gaussian:MAP} and~\ref{theo:nullspace2optimality}, shows that for Gaussian signals the optimal decoder~\eqref{eqn:MAP} leads to the instance optimality in expectation. 

\begin{corollary}
\label{corollary:nullspace2optimality:gaussian} 
For Gaussian signals $\bx \sim \mathcal{N}(\bzero, \bS)$, if an $M \times N$ sensing matrix $\Phi$ satisfies the null space property in expectation~\eqref{eqn:null:space:property:firstk} of order $k$ in $\|\cdot\|_1$, with constant $C_0/2$, or the MSE null space property~\eqref{eqn:null:space:property:MSE:firstk} of order $k$ with constant $C_0/4$, then the optimal and linear decoder $\Delta = \bS \Phi^T (\Phi \bS \Phi^T)^{-1}$ satisfies the instance optimality in expectation~\eqref{eqn:instance:optimality:firstk} in $\|\cdot\|_1$, or the MSE instance optimality~\eqref{eqn:instance:optimality:MSE:firstk}. 
\end{corollary}

\begin{proof}
It follows from Theorem~\ref{theo:gaussian:MAP} that the MAP decoder minimizes MAE and MSE among all the estimators for $\bx \sim \mathcal{N}(\bzero, \bS)$. Therefore its MAE and MSE are smaller than the ones generated by the decoder considered in Theorem~\ref{theo:nullspace2optimality}~\eqref{eqn:choice:decoder}. The latter satisfies the instance optimality, so is the former. 
\end{proof}

\subsubsection{From RIP to Null Space Property}
The Restricted Isometry Property (RIP) of a matrix measures its ability to preserve distances, and is related to the null space property in conventional CS~\cite{candes2006near, donoho2006compressed}. The new \textit{linear} RIP of order $k$ restricts the requirement of conventional RIP of order $k$ to a union of $k$-dimensional linear subspaces with consecutive supports:
\vspace{0ex}
\begin{definition}
Let $k \leq N$ be a positive integer. Let $\mathcal{K}_1$ define a linear subspace of functions with support in the first $k$ indices in $[1,N]$, $\mathcal{K}_2$ a linear subspace of functions with support in the next $k$ indices, and so on. The functions in the last linear subspace $\mathcal{K}_J$ defined this way may have support with less than $k$ indices. An $M \times N$ matrix $\Phi$ is said to have linear RIP of order $k$ with constant $\delta$ if 
\vspace{0ex}
\begin{equation}
\label{eqn:linear:RIP}
(1-\delta) \|\bx\|_2 \leq \|\Phi \bx\|_2 \leq (1+\delta) \|\bx\|_2,~~~\forall~\bx \in \cup_{j=1}^J \mathcal{K}_j.
\vspace{0ex}
\end{equation}
\end{definition}
The linear RIP is a special case of the block RIP~\cite{eldar2009robust},  with block sparsity one
and blocks having consecutive support of the same size.

The following theorem relates the linear RIP~\eqref{eqn:linear:RIP} of a matrix $\Phi$ to its null space property in expectation~\eqref{eqn:null:space:property:firstk}. 
\vspace{0ex}
\begin{theorem}
\label{theo:RIP2NullSpace}
Let $\bx \in \mathbb{R}^N$ be a random vector that follows a certain distribution. Let $\Phi$ be an $M \times N$ matrix that satisfies the linear RIP of order $2k$ with $\delta < 1$, and let $\Delta$ be a decoder. Let $\eta = \bx - \Delta (\Phi \bx)$.  Assume further that $E_{\bx, (\Phi)}|\eta[n]|$ decays in $n$: $E_{\bx, (\Phi)}|\eta[n+1]| < E_{\bx, (\Phi)}|\eta[n]|$, $\forall n < N-1$. Then $\Phi$ satisfies the null property in expectation of order $k$ in $\|\cdot\|_1$~\eqref{eqn:null:space:property:firstk}, with constant $C_0 = 1 + k^{1/2} \frac{1+\delta}{1-\delta}$.~\footnote{As in~\cite{cohen2009compressed}, the
result here is in the $l_1$ norm, while in the next section we will consider a natural extension of the RIP for SCS which can be studied in the $l_2$ norm, something possible for conventional CS only in a probabilistic setting, with one random sensing matrix independently drawn for each signal~\cite{cohen2009compressed}.} 
\vspace{0ex}
\end{theorem}
\begin{proof}
Let $K$ denote the set of first $k$ indices of the entries in $\eta$, $K_1$ the next $k$ indices, $K_2$ the next $k$ indices, etc. We have
\begin{eqnarray}
\|\eta_K\|_2 & \leq & \|\eta_{K \cup K_1} \|_2 \leq (1- \delta)^{-1} \|\Phi \eta_{K \cup K_1}\|_2 
= (1- \delta)^{-1} \| \sum_{j=2}^J \Phi \eta_{K_j}\|_2 \nonumber \\
&  \leq&  (1- \delta)^{-1} \sum_{j=2}^J \|  \Phi \eta_{K_j}\|_2 \leq  (1 + \delta) (1- \delta)^{-1} \sum_{j=2}^J \| \eta_{K_j}\|_2,  \nonumber 
\end{eqnarray}
where the second and last inequalities follow the linear RIP property of $\Phi$, the third inequality follows from the triangle equality, and the equality holds since $\eta \in \nsp(\Phi)$. Hence we have 
\begin{equation}
\label{eqn:theo:RIP:NP:1}
E\|\eta_K\|_2 \leq  (1 + \delta) (1- \delta)^{-1} \sum_{j=2}^J E \| \eta_{K_j}\|_2.
\end{equation}

Since $E|\eta[n+k]| \leq E|\eta[n]|$, we derive $E\|\eta_{K_{j+1}}\|_1 \leq E\|\eta_{K_{j}}\|_1$, so that
\begin{equation}
\label{eqn:theo:RIP:NP:2}
E \| \eta_{K_{j+1}}\|_2 \leq E \| \eta_{K_{j+1}}\|_1 \leq E\|\eta_{K_{j}}\|_1,
\end{equation}
where the first inequality follows from the fact that $\|\bx\|_2 \leq \|\bx\|_1, ~\forall \bx$. Inserting~\eqref{eqn:theo:RIP:NP:2} into~\eqref{eqn:theo:RIP:NP:1} gives
\begin{equation}
E\|\eta_K\|_2 \leq  (1 + \delta) (1- \delta)^{-1} \sum_{j=1}^{J-1} E \| \eta_{K_j}\|_1 \leq (1 + \delta) (1- \delta)^{-1} E \| \eta_{{K}^C}\|_1.
\end{equation}
By the Cauchy-Schwartz inequality $\|\eta_K\|_1 \leq k^{1/2} \|\eta_K\|_2$, we therefore obtain
\begin{equation}
E\|\eta\|_1 = E\|\eta_K\|_1 + E\|\eta_{K^C}\|_1 \leq \left( 1 + k^{1/2} \frac{1+\delta}{1-\delta}\right) E\|\eta_{K^C}\|_1,
\end{equation}
which verifies the null space property with constant $C_0$.
\end{proof}

For Gaussian signals $\bx \in \mathcal{N}(\bzero, \bS)$,  with $\Phi$ Gaussian or Bernoulli matrices, one realization drawn independently for each $\bx$, and with $\Delta$ the optimal decoder~\eqref{eqn:MAP}, the decay of  $E_{\bx, \Phi}|\eta[n]|$ assumed in Theorem~\ref{theo:RIP2NullSpace} is verified through Monte Carlo simulations.

\subsubsection{From Random Matrices to Linear RIP}
The next Theorem shows that Gaussian and Bernoulli matrices satisfy the conventional RIP for \textit{one} subspace with overwhelming probability. The linear RIP will be addressed after it.
\vspace{-1ex}
\begin{theorem}~\cite{achlioptas2003database,baraniuk2008simple}
\label{theo:RIP:prob:oneset}
Let $\Phi$ be a random matrix of size $M \times N$ drawn according to any distribution that satisfies the concentration inequality
\vspace{0ex}
\begin{equation}
\label{eqn:concentration:inequatliy}
\textrm{Pr}(|\|\Phi \bx\|_2^2 - \|\bx\|_2^2| \geq \epsilon \|\bx\|_2^2) \leq 2 e^{-M c_0(\delta/2)},~~~\forall~\bx \in \mathbb{R}^N,
\vspace{0ex}
\end{equation}
where $0 < \delta < 1$, and $c_0(\delta/2) > 0$ is a constant depending only on $\delta/2$. Then for any set $K \subset \{1,\ldots,N\}$ with $|K|=k<M$, we have the conventional RIP condition
\vspace{0ex}
\begin{equation}
\label{eqn:linear:RIP:oneset}
(1-\delta) \|\bx\|_2 \leq \|\Phi \bx\|_2 \leq (1+\delta) \|\bx\|_2,~~~\forall~\bx \in \mathcal{X}_K,
\vspace{0ex}
\end{equation}
where $\mathcal{X}_K$ is the set of all vectors in $\mathbb{R}^N$ that are zero outside of $K$, with probability greater than or equal to
$
1 - 2(12/\delta)^k e^{-c_0(\delta/2)M}.
$
Gaussian and Bernoulli matrices satisfy the concentration inequality~\eqref{eqn:concentration:inequatliy}.
\vspace{0ex}
\end{theorem}
The linear RIP of order $k$~\eqref{eqn:linear:RIP} requires that~\eqref{eqn:linear:RIP:oneset} holds for $N/k \le N$ subspaces. The next Theorem follows from Theorem~\ref{theo:RIP:prob:oneset} by simply multiplying by $N$ the probability that the RIP fails to hold for one subspace. 
\vspace{0ex}
\begin{theorem}
Suppose that $M$, $N$ and $0 < \delta < 1$ are given. Let $\Phi$ be a random matrix of size $M \times N$ drawn according to any distribution that satisfies the concentration inequality~\eqref{eqn:concentration:inequatliy}. Then there exist constants $c_1, c_2 > 0$ depending only on $\delta$ such that the linear RIP of order $k$~\eqref{eqn:linear:RIP} holds  with probability greater than or equal to $1- 2 N e^{-c_2 M}$ for $\Phi$ with the prescribed $\delta$ and $k \leq c_1 M$.
\vspace{0ex}
\end{theorem}

\begin{proof}
Following Theorem~\ref{theo:RIP:prob:oneset}, for a $k$-dimensional linear space $\mathcal{X}_K$, the matrix $\Phi$ will fail to satisfy~\eqref{eqn:linear:RIP:oneset} with probability $\leq 2(12/\delta)^k e^{-c_0(\delta/2)n}$.

The linear RIP requires that~\eqref{eqn:linear:RIP:oneset} holds for  at most $N$ such subspaces. Hence~\eqref{eqn:linear:RIP:oneset}  will fail to hold with probability 
\begin{equation}
\label{eqn:theorem:rm:RIP}
\leq 2 N(12/\delta)^k e^{-c_0(\delta/2)M} = 2 N e^{-c_0(\delta/2)M + k \log(12/\delta) }.
\end{equation}
Thus for a fixed $c_1 > 0$, whenever $k \leq c_1 M$, the exponent in the exponential on the right side of~\eqref{eqn:theorem:rm:RIP} is $\leq c_2 M$ provided that $c_2 \leq c_0(\delta/2) - c_1(1+\log(12/\delta))$. We can always choose $c_1 > 0$ small enough to ensure $c_2 > 0$. This proves that with a probability $1 - 2 N e^{-c_2 M}$, the matrix $\Phi$ will satisfy the linear RIP~\eqref{eqn:linear:RIP}. 
\end{proof}

Comparing with conventional CS, where the null space property requires that the RIP~\eqref{eqn:linear:RIP:oneset} holds for $\binom{N}{k}$ subspaces~\cite{baraniuk2008simple, candes2006near, donoho2006compressed}, the number of subspaces in the linear RIP~\eqref{eqn:linear:RIP} is sharply reduced to $N/k$ for Gaussian SCS, thanks to the coefficients pre-ordering and the linear estimation in consequence. Therefore with the same number of measurements $M$, the probability that a Gaussian or Bernoulli matrix $\Phi$ satisfies the linear RIP is substantially higher than that for the conventional RIP. Equivalently, given the same probability that $\Phi$ satisfies the linear RIP or the conventional RIP of order $k$, the required number of measurements for the linear RIP is $M \sim \mathcal{O}(k)$, substantially smaller than the $M \sim \mathcal{O}(k \log(N/k))$ required for the conventional RIP. Similar improvements have been obtained with model-based CS that assumes structured sparsity on the signals~\cite{baraniuk2010model}. 

With the results above, we have shown that for Gaussian signals, with sensing matrices satisfying the linear RIP~\eqref{eqn:linear:RIP} of order $2k$, for example Gaussian or Bernoulli matrices with $\mathcal{O}(k)$ rows, with overwhelming probability, and with the optimal and linear decoder~\eqref{eqn:MAP}, SCS leads to the instance optimality in expectation  of order $k$ in $\|\cdot\|_1$~\eqref{eqn:instance:optimality:firstk}, with constant $C_0 = 2(1 +  k^{1/2} \frac{1+\delta}{1-\delta})$. $k^{1/2}$ is typically small by the definition of CS. 

\subsection{Performance Bounds with RIP in Expectation}
\label{sec:RIP:expect}
This section shows that with an {\it RIP in expectation}, a matrix isometry property more adapted to SCS, the Gaussian SCS MSE instance optimality~\eqref{eqn:instance:optimality:MSE:firstk} of order $k$ and constant $C_0$, holds in the $l_2$ norm with probability one for {\it any} matrix. $C_0$ has a closed-form and can be easily computed numerically.
\vspace{0ex}
\begin{definition}
\label{def:RIP:expect}
Let $\bx \in \mathbb{R}^N$ be a random vector that follows a certain distribution. Let  $\Phi$ be an $M \times N$ sensing matrix and let $\Delta$ be a decoder. Let $\eta = \bx - \Delta (\Phi \bx)$. $\Phi$ in  $(\Phi, \Delta)$ is said to have RIP in expectation in $K$ with constant $c_K$ if 
\vspace{0ex}
\begin{equation}
\label{eqn:RIP:expect}
{E_{\bx, (\Phi)} \|\Phi \eta_K\|_2^2} = c_K {E_{\bx, (\Phi)} \|\eta_K\|_2^2},~\textrm{where}~\eta = \bx - \Delta (\Phi \bx),
\vspace{0ex}
\end{equation}
where $K \subset \{1, \ldots, N\}$, $\eta_K \in \mathbb{R}^N$ is the signal $\eta$ restricted to $K$ ($\eta_K[n]=\eta[n],~\forall~n \in K$, and $0$ otherwise), and the expectation is with respect to $\bx$, and to $\Phi$ if one random $\Phi$ is drawn independently for each $\bx$.
\vspace{0ex}
\end{definition}

The conventional RIP is known to be satisfied  only by some random matrices, Gaussian and Bernoulli matrices for example, with high probability. For a given matrix, checking the RIP property is however NP-hard~\cite{baraniuk2008simple}. By contrast, the constant of the RIP in expectation~\eqref{eqn:RIP:expect} can be measured for \textit{any} matrix via a fast Monte Carlo simulation, the quick convergence guaranteed by the concentration of measure~\cite{talagrand1996new}. The next proposition, directly following from~\eqref{eqn:MSE:GSCS} and~\eqref{eqn:MSE:GSCS:randPhi}, further shows that for Gaussian signals, the RIP in expectation has its constant in a closed form. 
\vspace{0ex}
\begin{prop}
\label{prop:RIP:expect:gassian}
Assume $\bx \sim \mathcal{N}(\bzero, \bS)$, $\Phi$ is an $M \times N$ sensing matrix and $\Delta$ is the optimal and linear decoder~\eqref{eqn:MAP}. Then $\Phi$ in  $(\Phi, \Delta)$ satisfies the RIP in expectation in $K$, 
\vspace{0ex}
{\small
\begin{equation}
\label{eqn:RIP:gauss}
  {(E_\Phi) \left[Tr\left( \Phi \bR_K \bS \bR_K^T \Phi^T - \Phi \bR_K \bS \Phi^T (\Phi \bS \Phi^T)^{-1} \Phi \bS \bR_K^T \Phi^T\right) \right]} \nonumber \vspace{-1ex}
 \end{equation}
 \begin{equation} 
   = c_K
{(E_\Phi) \left[ Tr\left(\bR_K \bS \bR_K^T - \bR_K \bS \Phi^T (\Phi \bS \Phi^T)^{-1} \Phi \bS \bR_K^T\right) \right]} \vspace{0ex},
\end{equation}
}where $\bR_K$ is an $N \times N$ extraction matrix giving $\eta_K = \bR_K \eta$, i.e., $\bR_K(i,i)=1$, $\forall i \in K$, all the other entries being zero. The expectation with respect to $\Phi$ is calculated if one random $\Phi$ is drawn independently for each $\bx$. 
\vspace{0ex}
\end{prop}

\begin{proof}
Let $\eta = \bx - \Delta \Phi \bx =  \bx - \bS \Phi^T (\Phi \bS \Phi^T)^{-1} \Phi \bx$, which follows from the MAP estimation~\eqref{eqn:MAP}.~\eqref{eqn:RIP:gauss} is derived by calculating the covariance matrices 
$\Sigma_{\Phi \eta_K} = E\left[  \Phi \bR_K \eta (\Phi \bR_K \eta)^T\right]$ of $\Phi \eta_K = \Phi \bR_K \eta$, and $\Sigma_{\eta_K} = E\left[  \bR_K \eta (\bR_K \eta)^T\right]$ of $\eta_K = \bR_K \eta$,
and using the fact that the trace of a covariance matrix yields the average energy of the underlying random vector. 
\end{proof}

The next Theorem shows that the RIP in expectation leads to the MSE null space property holding in equality. 
\vspace{0ex}
\begin{theorem}
\label{theo:RIP:expect:null:space}
Let $\bx \in \mathbb{R}^N$ be a random vector that follows a certain distribution,  $\Phi$ an $M \times N$ sensing matrix, and $\Delta$ a decoder. 
Assume ${E_{\bx, (\Phi)}  \|\eta_{K}\|_2^2} \neq 0$ and ${E_{\bx, (\Phi)}  \|\eta_{K^C}\|_2^2} \neq 0$, for some $K \subset \{1, \ldots, N\}$. Assume that $\Phi$ in $(\Phi, \Delta)$ has the RIP in expectation in $K$ with constant $a_K > 0$, and in $K^C =  \{1, \ldots, N\}\backslash K$  with constant $b_K > 0$:
\vspace{0ex}
\begin{equation}
\label{eqn:RIP:expt}
\frac{E_{\bx, (\Phi)} \|\Phi \eta_{K}\|_2^2}{E_{\bx, (\Phi)}  \|\eta_{K}\|_2^2} = a_K,~~~\frac{E_{\bx, (\Phi)}  \|\Phi \eta_{K^C}\|_2^2}{E_{\bx, (\Phi)}  \|\eta_{K^C}\|_2^2} = b_K,
~\textrm{where}~\eta = \bx - \Delta \Phi \bx, 
\end{equation}
where $K \subset \{1, \ldots, N\}$, and $\eta_K \in \mathbb{R}^N$ is the signal $\eta$ restricted to $K$ ($\eta_K[n]=\eta[n],~\forall~n \in K$, and $0$ otherwise).
Then $\Phi$ satisfies
\vspace{0ex}
\begin{equation}
\label{eqn:null:space:equality1}
E_{\bx, (\Phi)}\|\eta\|_2^2 = C_0 E_{\bx, (\Phi)}\|\eta_{K^C}\|_2^2,
\vspace{0ex}
\end{equation}
where $C_0 = 1 + {b_K}/{a_K}$. In particular, if $K = \{1, \ldots, k\}$, with $1 \leq k \leq N$, then  $\Phi$ satisfies the MSE null space property of order $k$, which holds with equality,
\vspace{0ex}
\begin{equation}
\label{eqn:null:space:equality}
E_{\bx, (\Phi)}\|\eta\|_2^2 = C_0 \sigma_k(\{\eta\})_2^2.
\vspace{0ex}
\end{equation}
\end{theorem}
\begin{proof}
We derive~\eqref{eqn:null:space:equality1} by 
$$
\frac{E_{\bx, (\Phi)}\|\eta\|_2^2}{E_{\bx, (\Phi)}\|\eta_{K^C}\|_2^2} = 1 + \frac{E_{\bx, (\Phi)}\|\eta_K\|_2^2}{E_{\bx, (\Phi)}\|\eta_{K^C}\|_2^2}  = 1 + \frac{E_{\bx, (\Phi)}\|\Phi \eta_{K}\|_2^2/a_k}{E_{\bx, (\Phi)}\|\Phi \eta_{K^C}\|_2^2/b_k} = 1 + \frac{b_k}{a_k},
$$
where the second equality follows from the RIP in expectation~\eqref{eqn:RIP:expt} and the last equality holds  because $\Phi \eta_K = \Phi \eta_{K^C}$ since $\eta=\eta_K + \eta_{K^C} \in \nsp(\Phi)$.~\eqref{eqn:null:space:equality} is obtained by inserting~\eqref{eqn:err:best:k:linear} in~\eqref{eqn:null:space:equality1}.\end{proof}

Following Corollary~\ref{corollary:nullspace2optimality:gaussian}, the MSE null space property constant $C_0$ indicates the upper bound of the SCS reconstruction error relative to the best $k$-term linear approximation. Let us check $C_0$ of different sensing matrices in SCS for Gaussian signals $\bx \in \mathbb{R}^N \sim \mathcal{N}(\bzero, \bS)$, assuming that the eigenvalues of $\bS$ follow a power decay~\eqref{eqn:eigvalue:power:decay} with typical values $\alpha=3$ and $N=64$. Gaussian, Bernoulli and random subsampling matrices $\Phi$ of size $M \times N$ are considered, and the optimal and linear decoder $\Delta$~\eqref{eqn:MAP} is applied to reconstruct the signals. For each matrix distribution, a different random matrix realization $\Phi$ is applied to sense each signal $\bx$. Note that since the random subsampling matrix $\Phi$, each row containing one entry with value 1 at a random position and 0 otherwise, has the maximal coherence with the canonical basis, this matrix is not suitable for directly sensing $\bx$~\cite{candes2007sparsity}, and is replaced by $\Phi \dict$ in the simulation, with $\dict$ a DCT basis having low coherence with $\Phi$. 

Monte Carlo simulations are performed to calculate the RIP constants $a_K$ and $b_K$~\eqref{eqn:RIP:expt}. Figure~\ref{fig:RIP:expect} (a) plots $C_0 = 1 + {b_K}/{a_K}$, with a typical value $k=10$ ($k/N=5/32$), for different values of $M$. When the number $M$ of SCS measurements increases, the reconstruction error of SCS decreases, resulting in a smaller ratio over the best $k$-term linear approximation error with a fixed $k$. Gaussian and Bernoulli matrices lead to similar $C_0$ values, slightly smaller than that of random subsampling matrices. Figure~\ref{fig:RIP:expect} (b) plots $C_0$, as a function of $k$, with $M=k$. Gaussian and Bernoulli matrices lead to similar $C_0 \approx 4.5$ that varies little with $k$, in line with the results obtained in Section~\ref{subsec:SCS:numeric:analysis} (Figure~\ref{fig:MSE:scs:vs:bestk}-(c)). For random subsampling matrices $C_0$ slowly increases, almost linearly, and is equal to $5.5$ for a typical value $k=10$, about 20\% larger than that of Gaussian and Bernoulli matrices. The small $C_0$ values indicate that the SCS reconstruction error is tightly upper bounded by a constant times the best $k$-term approximation error. 

\begin{figure}[htbp]
\vspace{0ex}
\begin{center}
\begin{tabular}{cc}
\includegraphics[width=6cm]{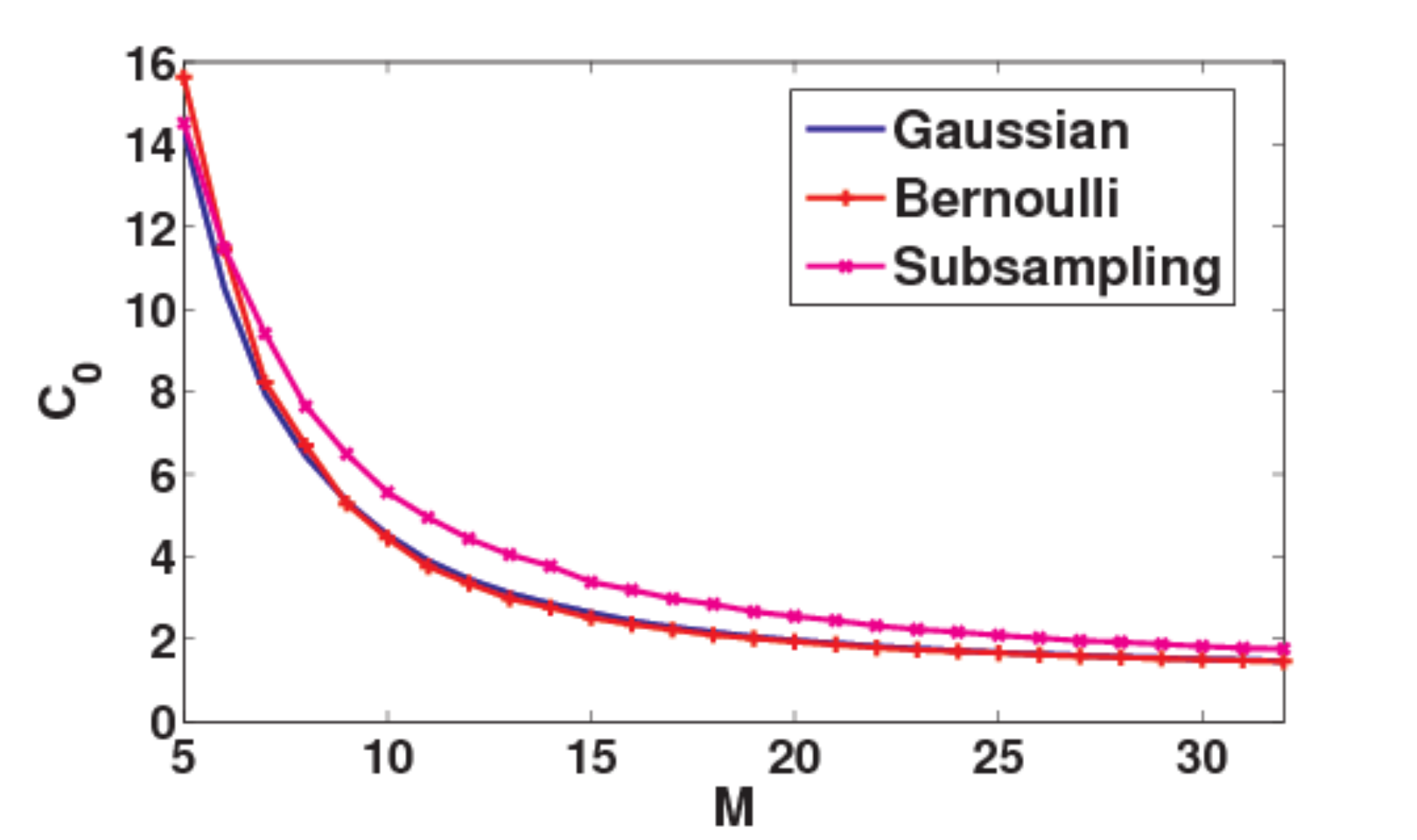}  &
\includegraphics[width=6cm]{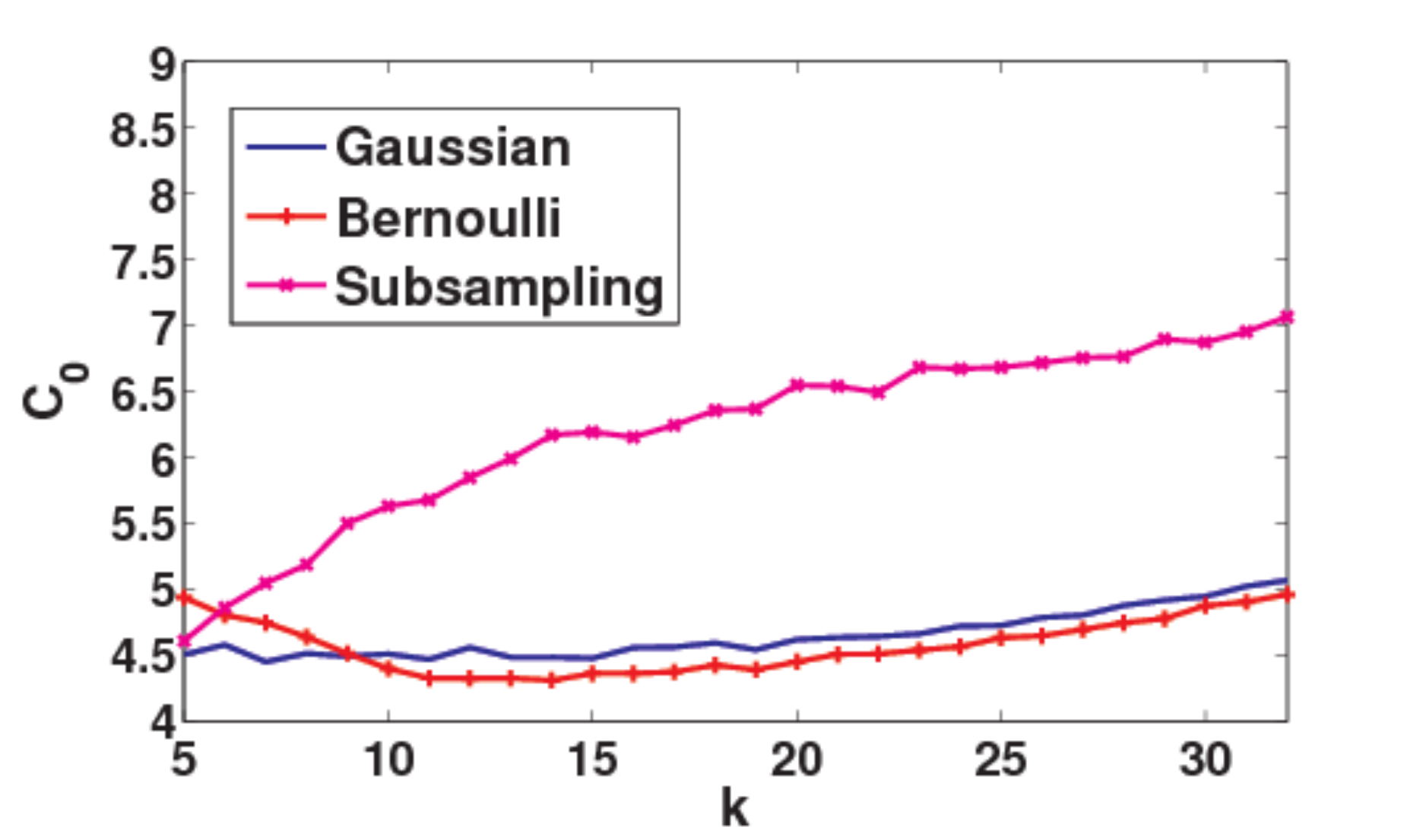} \vspace{0ex} \\
\textbf{\small(a)} & \textbf{\small(b)} \vspace{-1ex} \\
\end{tabular}
\end{center}
\vspace{0ex}
\caption{The MSE null space property constant $C_0$~\eqref{eqn:null:space:equality} of Gaussian, Bernoulli, and random subsampling matrices, as a function of $M$, with a fixed $k=10$ (left), and of $k$ with $M=k$ (right). The signal dimension is $N=64$.} \label{fig:RIP:expect}
\vspace{0ex}
\end{figure}
From Corollary~\ref{corollary:nullspace2optimality:gaussian} and Theorem~\ref{theo:RIP:expect:null:space}, we obtain the next concluding Theorem, which shows that for \textit{any} sensing matrix, the error of Gaussian SCS is upper bounded by a constant times the best $k$-term linear approximation with probability one, and the bound constant can be efficiently calculated.
\vspace{0ex}
\begin{theorem}
\label{theo:SCS:with:RIP:expect}
Assume $\bx \sim \mathcal{N}(\bzero, \bS)$. Let $\Phi$ be an $M \times N$ sensing matrix and $\Delta$ the optimal and linear decoder~\eqref{eqn:MAP}. Then $\Phi$ satisfies the MSE instance optimality of order $k$~\eqref{eqn:instance:optimality:MSE:firstk} with constant $C_0 = 4(1 + {b_K}/{a_K})$, $a_K$ and $b_K$ given in~\eqref{eqn:RIP:expt}, and $K = \{1, \ldots, k\}$.\vspace{0ex}
\end{theorem}

Theorem~\ref{theo:SCS:with:RIP:expect}, together with the performance comparison of linear and nonlinear approximation for Gaussian signals described in Section~\ref{subsec:linear:vs:nonlinear}, show that for signals following a Gaussian distribution with fast eigenvalue decay, the average error of SCS using $k$ measurements is tightly upper bounded by that of the best $k$-term approximation. 

\section{Compressed Sensing Model Selection with GMMs}
\label{sec:SCS:GMM:analysis}
Section~\ref{sec:SCS:GM} shows tight error bounds of SCS for signals following a Gaussian distribution with fast eigenvalue decay. A single Gaussian distribution, however, is too simplistic for modeling most real signals. Assuming multiple Gaussian distributions and that each signal follows one of them, Gaussian mixture models (GMMs) provide more precise signal descriptions. It has been shown that algorithms based on GMMs lead to results in the ballpark of the state-of-the-art in various signal inverse problems, for different types of real data including images and ranking score matrices~\cite{leger2010Matrix, yu2010PLE}. GMMs have also been used to model color distributions~\cite{rother2004grabcut} and for clustering~\cite{duda2000pattern}, among many satisfactory applications with these models.

This section first introduces a piecewise linear decoder for  GMM-based SCS, which essentially consists of estimating a signal using each Gaussian model included in the GMM and then selecting the best model. At the heart of the GMM-based SCS decoder is the model selection. The rest of the section analyzes the accuracy of the model selection in terms of the GMM properties and the number of the measurements. As correct Gaussian models are selected, the SCS performance bounds described in Section~\ref{sec:SCS:GM} apply.  

\subsection{Piecewise Linear Decoder}
\label{subsec:algorithm:SCS:GMM}

GMMs describe signals with a mixture of Gaussian distributions.  Assume there exist $J$ Gaussian distributions $\{\mathcal{N} (\mu_j, \Sigma_j)\}_{1 \leq j \leq J}$, parametrized by their means $\mu_j$ and covariances $\Sigma_j$. To simplify the notation, we assume without loss of generality that the Gaussians have zero means ${\mu}_j = \mathbf{0}$, $\forall j$, as one can always center the signals  with respect to the means. GMM assumes that each signal $\bx \in \mathbb{R}^N$ is independently drawn from one of these Gaussians with an unknown index $j \in [1, J]$, whose probability density function is
\begin{equation}
\label{eqn:multivariate:gaussian}
f(\bx) = \frac{1}{(2\pi)^{N/2} |\Sigma_{j}|^{1/2}} \exp\left({-\frac{1}{2} \bx^T \Sigma_{j}^{-1} \bx}\right).
\end{equation}

To decode a measured signal $\by = \Phi \bx$, the GMM-based SCS decoder estimates the signal $\tilde{\bx}$ and selects the Gaussian model $\tilde{j}$ by maximizing the log a-posteriori probability 
\begin{equation}
\label{eqn:MAP:x:k}
(\tilde{\bx}, \tilde{j}) = \arg \max_{\bx, j} \log f(\bx | \by, \Sigma_j).
\end{equation}
\eqref{eqn:MAP:x:k} is calculated by first computing
the linear MAP decoder~\eqref{eqn:MAP:Sigma} using each of the Gaussian models, 
\begin{equation}
\label{eqn:MAP:Sigma:K} 
\tilde{\bx}_j = \Delta_j (\Phi \bx) = \underbrace{\Sigma_j \Phi^T (\Phi \Sigma_j \Phi^T)^{-1}}_{\Delta_j} (\Phi \bx),~~~\forall  1 \leq j \leq J,
\vspace{0ex}
\end{equation}
and then selecting a best model $\tilde{j}$ that maximizes the log  a-posteriori probability  among all the models~\cite{yu2010PLE}
\begin{equation}
\label{eqn:GMM:model:selection} 
\tilde{j} = \arg \max_{1 \leq j \leq J} -\frac{1}{2} \left(\log |\Sigma_{j}| + \tilde{\bx}_j^T \Sigma_{j}^{-1} \tilde{\bx}_j\right),
\end{equation}
whose corresponding decoder $\Delta_{\tilde{j}}$ implements a piecewise linear estimate: 
\begin{equation}
\label{eqn:best:decoder} 
\tilde{\bx} = \tilde{\bx}_{\tilde{j}}  = \Delta_{\tilde{j}} (\Phi \bx). 
\end{equation}

The model selection~\eqref{eqn:GMM:model:selection} is at the heart of the GMM-based SCS.\footnote{Correct model/class selection from compressed measurements is at the core of numerous applications beyond signal reconstruction, see for example~\cite{chechik2007max} and references therein.} To better understand it, we concentrate next in a simple case, where the GMM involves $J=2$ Gaussian distributions $\mathcal{N}(\mathbf{0}, \Sigma_1)$ and $\mathcal{N}(\mathbf{0}, \Sigma_2)$ that have the same ``shape'' and ``size'', but different ``orientation,'' i.e., the two covariance matrices have the same eigenvalues, but different PCA bases:
\begin{equation}
\label{eqn:Sigma12:SVD}
\Sigma_1 = \bB_1 \bS \bB_1^T~~~\textrm{and}~~~\Sigma_2 = \bB_2 \bS \bB_2^T,
\end{equation}
with $\bB_1$ and $\bB_2$ the PCA bases of the two Gaussian distributions, and  $\bS = \mathrm{diag}(\lambda_{1}, \ldots, \lambda_{N})$ a diagonal matrix, whose diagonal elements $\lambda_1 \geq \lambda_2 \geq \ldots \geq \lambda_N$ are the sorted eigenvalues. It follows directly that $|\Sigma_1| = |\Sigma_2|$. This will be used next.

\subsection{Oracle Model Selection}
\label{subsec:oracle:model:selection}
Let us first study the model selection in an oracle situation, where the underlying signals $\bx$ are assumed to be known and, without loss of generality, to follow the first Gaussian distribution $\bx \sim \mathcal{N}(\mathbf{0}, \Sigma_1)$. Recall that $|\Sigma_1| = |\Sigma_2|$ is assumed. The probability of correct oracle model selection~\eqref{eqn:GMM:model:selection}  that assigns $\bx$ to the first Gaussian distribution $\mathcal{N}(\mathbf{0}, \Sigma_1)$,
\begin{equation}
\label{eqn:prob:clustering:error}
P_c^o = \int_{\bx^T \Sigma_1^{-1} \bx < \bx^T \Sigma_2^{-1} \bx } f_1(\bx) d\bx = \int \textrm{sign}\left(\bx^T \Sigma_2^{-1} \bx - \bx^T \Sigma_1^{-1} \bx \right) f_1(\bx) d\bx ,
\end{equation}
where $f_1(\bx) = \frac{1}{(2\pi)^{N/2} |\Sigma_1|^{1/2}} \exp\left({-\frac{1}{2} \bx^T \Sigma_1^{-1} \bx}\right)
$, will be studied as a function of the relationship between $\bB_1$ and $\bB_2$, the decay rate of the eigenvalues, and the signal dimension $N$. 

\subsubsection{KL Divergence} To better understand~\eqref{eqn:prob:clustering:error}, let us first check the Kullback-Leibler (KL) divergence from the first Gaussian distribution to the second
\begin{eqnarray}
\label{eqn:KL:gaussians} 
D_{KL} & = & \frac{1}{2}\int \left(\bx^T \Sigma_2^{-1} \bx - \bx^T \Sigma_1^{-1} \bx \right) f_1(\bx) d\bx \\
& = & \frac{1}{2}\textrm{Tr}(\Sigma_2^{-1} \Sigma_1 - \bI_N)  = \frac{1}{2} (\textrm{Tr}(\Sigma_2^{-1} \Sigma_1) - N),\label{eqn:KL:gaussians:2} 
\end{eqnarray}
where $\bI_N$ denotes the $N \times N$ identify matrix, and the second equality holds since $E[\bx^T \bA \bx] = \textrm{Tr}(\bA \Sigma)$ if $\bx \sim \mathcal{N}(\mathbf{0}, \Sigma)$~\cite{petersen2006matrix}. Comparing~\eqref{eqn:KL:gaussians} and~\eqref{eqn:prob:clustering:error}, we observe that $D_{KL}$ is monotonic relative to $P_c^o$. Analyzing the behavior of 
$D_{KL}$ as a function of the two Gaussians thus helps to understand that of $P_c^o$. 

Inserting~\eqref{eqn:Sigma12:SVD} into~\eqref{eqn:KL:gaussians:2} leads to 
\begin{equation}
\label{eqn:KL:gaussians3} 
D_{KL} = \frac{1}{2}(\textrm{Tr}(\bB_2 \bS^{-1} \bB_2^T \bB_1 \bS \bB_1^T) - N) = \frac{1}{2}(\textrm{Tr}(\bC \bS \bC^T \bS^{-1}) - N),
\end{equation}
where $\bC = \bB_2^T \bB_1$, and the second equality follows from the cyclic permutation invariance property of the trace $\textrm{Tr}(\bA\bB\bC) = \textrm{Tr}(\bC\bB\bA)$. Note that $\bC$ is an orthogonal matrix: $\bC^T \bC = \bI_N$. Maximizing $D_{KL}$ with respect to $\bB_1$ and $\bB_2$ is therefore equivalent to maximizing $\textrm{Tr}(\bC \bS \bC^T \bS^{-1})$ with respect to $\bC$. The following lemma shows that in dimension two, $D_{KL}$ is maximized when the first principal directions of the two Gaussians are orthogonal, and moreover, the maximum divergence increases as the Gaussians become more anisotropic. 

\begin{lemma} 
\label{lemma:KLdiv:2DGaussian}
Let $\bB_1$ and $\bB_2$ be respectively the PCA bases ($\Sigma_1 = \bB_1 \bS \bB_1^T$ and $\Sigma_2 = \bB_2 \bS \bB_2^T$) of two centered 2D Gaussian distributions $\mathcal{N}(\mathbf{0}, \Sigma_1)$ and $\mathcal{N}(\mathbf{0}, \Sigma_2)$, and $\bS = \left[ \begin{array}{cc}
\lambda_1 & 0 \\
0 & \lambda_2 \\
 \end{array} \right]$, with $\lambda_1>0$ and $\lambda_2>0$ their common eigenvalues. The KL divergence from the first Gaussian distribution to the second~\eqref{eqn:KL:gaussians}  has a maximum value 
 \begin{equation}
 \label{eqn:KL:gaussians:2D}
 D_{KL}^{\max} = \max_{\bB_1, \bB_2} D_{KL} = \frac{1}{2} \left( \frac{\lambda_2}{\lambda_1} + \frac{\lambda_1}{\lambda_2} \right),
 \end{equation}
which is obtained when $\bB_2^T \bB_1 = \left[ \begin{array}{cc}
0 & 1 \\
1 & 0 \\
 \end{array} \right]$. 
 
Let the determinant of the covariance matrices $|\Sigma_1| = |\Sigma_2| = \lambda_1 \lambda_2$ further be assumed given. Then $D_{KL}^{\max}$ is minimized as $\lambda_1 = \lambda_2$, and it increases as the ratio between $\lambda_1$ and $\lambda_2$ increases. 
\end{lemma}
\begin{proof}
The first part of the lemma can be easily checked by maximizing $D_{KL}$ in~\eqref{eqn:KL:gaussians3} with respect to the 2D orthogonal matrix $\bC = \bB_2^T \bB_1$ and writing $ \bC = \left[ \begin{array}{cc}
\cos \theta & -\sin \theta \\
\sin \theta & \cos \theta \\
 \end{array} \right]$.  The second part is verified via a direct observation of~\eqref{eqn:KL:gaussians:2D}. 
\end{proof}

Figure~\ref{fig:KL:Prob:2Gaussians}-(a) plots $D_{KL}$  as a function of the angle $\theta$ between the first principal components of the two 2D Gaussians going from $5^\circ$ to $90^\circ$,  with different eigenvalue ratios $\lambda_1/\lambda_2$ from 5 to 100. As indicated by 
Lemma~\ref{lemma:KLdiv:2DGaussian}, given $\lambda_1/\lambda_2$, $D_{KL}$ increases as $\theta$ increases. At a given $\theta$, larger $\lambda_1/\lambda_2$ leads to larger $D_{KL}$. 

The analysis in higher dimension is more difficult, however, one can check via a greedy optimization that  
\begin{equation}
\label{eqn:KL:greedy:optimal:C}
\bC = \bB_2^T \bB_1 = \begin{bmatrix}
       0& \cdots & \cdots & 0 & 1           \\   
       \vdots& \cdots & \iddots & 1& 0           \\    
       \vdots& \iddots & \iddots & \iddots& \vdots          \\    
       0& 1 & \iddots & \cdots& \vdots           \\    
       1& 0 & \cdots & \cdots& 0           \\    
     \end{bmatrix},
\end{equation}
with ones along the anti-diagonal, and zeros elsewhere, gives a local maximum of~\eqref{eqn:KL:gaussians3}.
In other words, the two Gaussians being ``orthogonal'' one another, i.e.,  the alignment of the first principal component of one Gaussian to the last principal component of the other, the second principal component of the former to the second to last principal component of the latter, and so on, leads to a local maximization of~\eqref{eqn:KL:gaussians3}. This can be observed by inserting
$$
\bC = 
\begin{bmatrix}
C_{11} &  \ldots & C_{1N} \\
\vdots & \ddots & \vdots \\
C_{N1} &  \ldots & C_{NN} \\
\end{bmatrix}
$$
in~\eqref{eqn:KL:gaussians3}, which gives
\begin{equation}
\label{eqn:KL:gaussians:4} 
D_{KL} = \frac{1}{2}(\sum_{m=1}^N \frac{1}{\lambda_m} \sum_{n=1}^N \lambda_n C_{mn}^2
- N). 
\end{equation}
A greedy maximization of~\eqref{eqn:KL:gaussians:4}  with respect to $\bC$ is calculated by scanning $\bC$ row by row from bottom to top, observing that $1/\lambda_m$ decreases as $m$ goes from $N$ to $1$, and at each $m$-th row scanning $C_{mn}$ from left to right, observing that $\lambda_n$ increases as $n$ goes from $1$ to $N$, taking into account the constraint $\bC^T \bC = \bI_N$. A similar observation of~\eqref{eqn:KL:gaussians:4} shows that when $D_{KL}$ is at the local maximum with $\bC$ equal to~\eqref{eqn:KL:greedy:optimal:C}, its value increases as the eigenvalues decay faster from $\lambda_1$ to $\lambda_N$. 

\subsubsection{Correct Model Seletion Probability} 
The probability of correct oracle model selection $P_c^o$~\eqref{eqn:prob:clustering:error} is now evaluated via Monte Carlo simulations. Figure~\ref{fig:KL:Prob:2Gaussians}-(b) plots $P_c^o$ as a function the angle $\theta$ between the first principal components of the two 2D Gaussians going from $5^\circ$ to $90^\circ$, with different eigenvalue ratios $\lambda_1/\lambda_2$ from 5 to 100.  As illustrated in Figure~\ref{fig:KL:Prob:2Gaussians}, $P_c^o$ shows a behavior similar to the KL-divergence $D_{KL}$ as a function of $\theta$ and of $\lambda_1/\lambda_2$: Given $\lambda_1/\lambda_2$, $P_c^o$ increases as $\theta$ increases; at a given $\theta$, larger $\lambda_1/\lambda_2$ leads to larger $P_c^o$. In contrast to $D_{KL}$, whose value is roughly proportional to $\lambda_1/\lambda_2$ (as $\lambda_1 \gg \lambda_2$), $P_c^o$ presents a saturation effect: $\lambda_1/\lambda_2$ values larger than about 40 lead to comparable $P_c^o$ that increases rapidly as a function of $\theta$, converging to a high value around 0.9; for $\lambda_1/\lambda_2$ smaller than about 40, on the other hand, $P_c^o$ reduces quickly as $\lambda_1/\lambda_2$ shrinks towards 1.

\begin{figure}[htbp]
\vspace{-13ex}
\begin{center}
\begin{tabular}{cc}
\includegraphics[width=6cm]{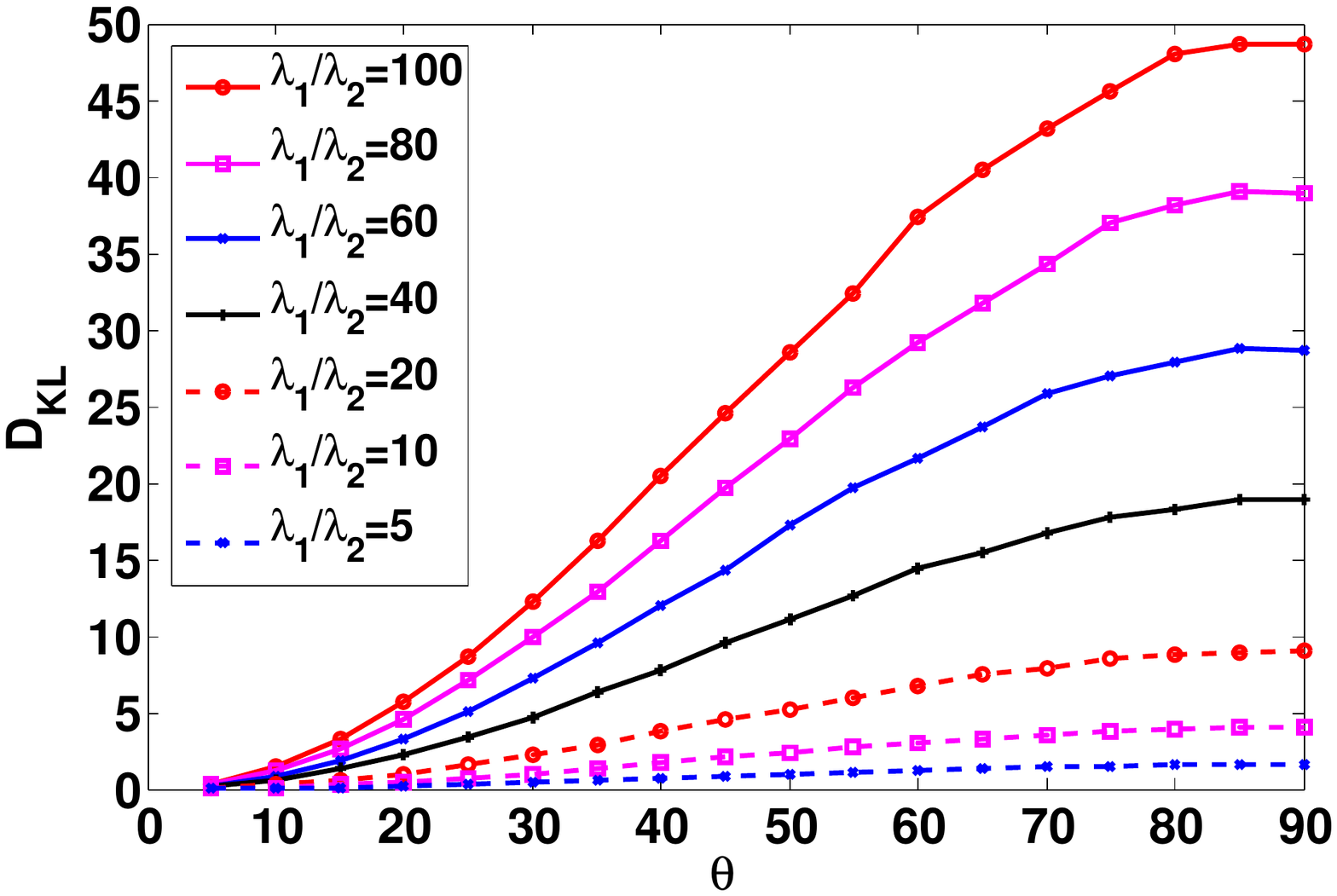}  &
\includegraphics[width=6cm]{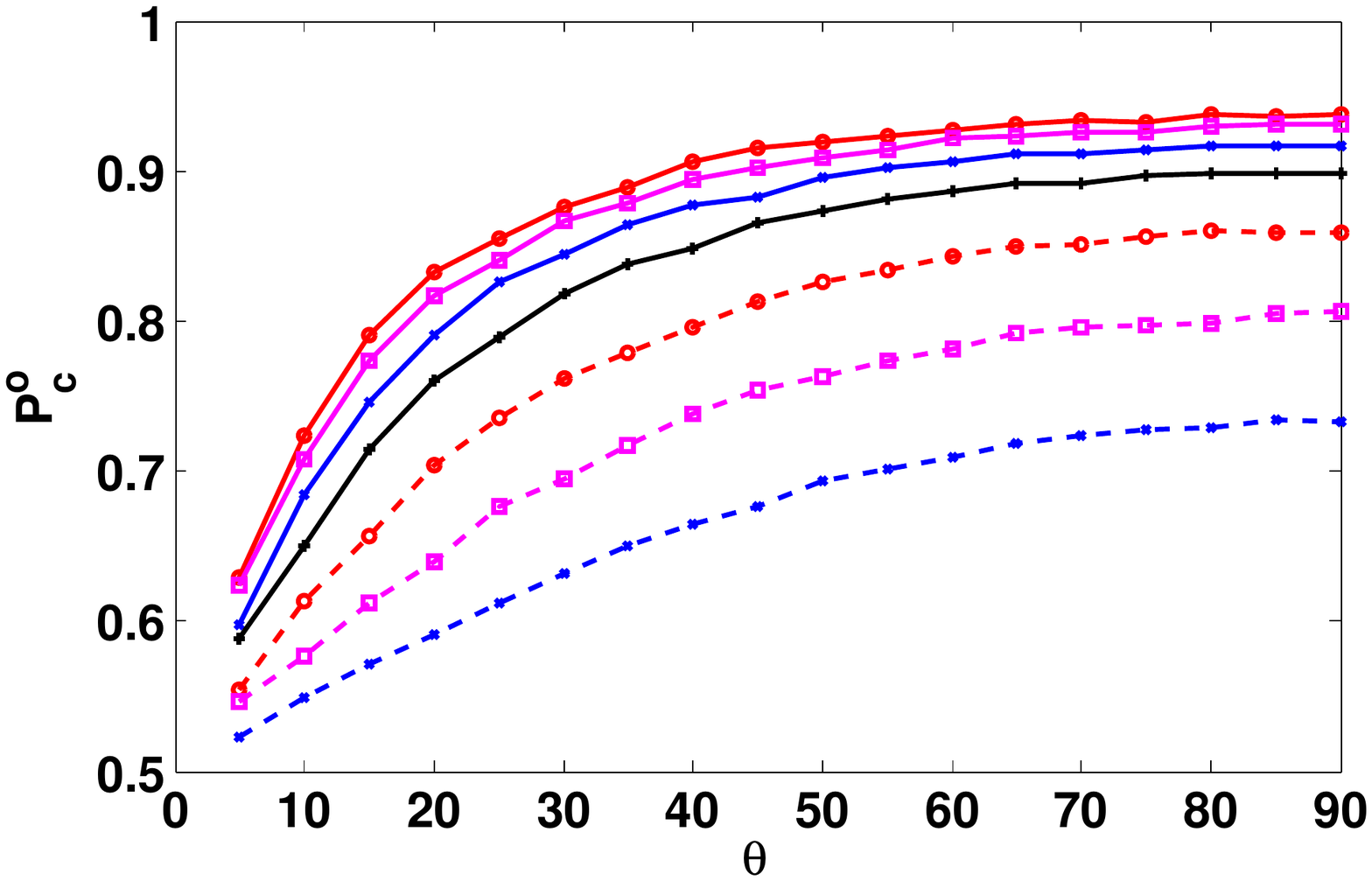} \vspace{-12ex} \\
\textbf{\small(a)} & \textbf{\small(b)} \vspace{-1ex} \\
\end{tabular}
\end{center}
\vspace{0ex}
\caption{(a) The KL-divergence~\eqref{eqn:KL:gaussians}  between two 2D Gaussians, as a function the angle $\theta$ between the first principal components of the two Gaussians going from $5^\circ$ to $90^\circ$, with different eigenvalue ratios $\lambda_1/\lambda_2$ from 5 to 100. (b) The same for the probability of correct oracle model selection $P_c^o$~\eqref{eqn:prob:clustering:error}.} \label{fig:KL:Prob:2Gaussians}
\vspace{0ex}
\end{figure}

Figure~\ref{fig:Prob:2Gaussians:HighDim} shows the probability of correct oracle model selection $P_c^o$~\eqref{eqn:prob:clustering:error} in higher dimensions, under the condition that~\eqref{eqn:KL:greedy:optimal:C} holds, i.e., the two Gaussians are ``orthogonal.'' A power decay of the eigenvalues~\eqref{eqn:eigvalue:power:decay} is assumed in the Monte Carlo simulations. In different signal dimensions $N$ from $2$ to $20$, $P_c^o$ as a function of the eigenvalue decay parameter $\alpha$ is plotted. For a given dimension, $P_c^o$ increases as $\alpha$ increases, i.e., as the eigenvalues decay faster so that the Gaussians are more anisotropic.  It is important to remark that,  with the same $\alpha$, $P_c^o$ rapidly increases as the signal dimension $N$ increases, which shows that anisotropic Gaussians with their energy concentrated in the first few dimensions are more separate in higher dimension. 

\begin{figure}[htbp]
\vspace{-13ex}
\begin{center}
\begin{tabular}{c}
\includegraphics[width=7cm]{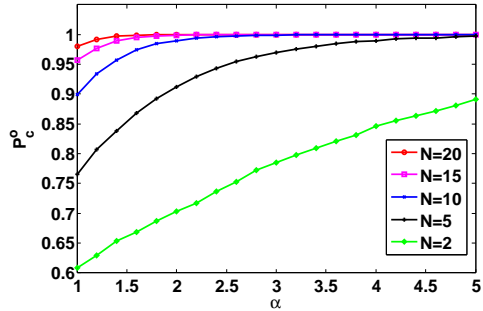} \vspace{-13ex} \\
\end{tabular}
\end{center}
\vspace{0ex}
\caption{The probability of correct oracle model selection $P_c^o$~\eqref{eqn:prob:clustering:error}  between two  Gaussians, as a function of the eigenvalue decay parameter $\alpha$ from 1 to 5, for different signal dimensions $N$ from 2 to 20. The two Gaussians satisfy~\eqref{eqn:KL:greedy:optimal:C}.} \label{fig:Prob:2Gaussians:HighDim}
\vspace{0ex}
\end{figure}

\subsection{Model Selection and Signal Reconstruction}
In SCS, the model selection~\eqref{eqn:GMM:model:selection} is calculated with the decoded signals~\eqref{eqn:MAP:Sigma:K} and not the ideal ones. Assume without loss of generality that the signals follow the first Gaussian distribution $\bx ~\sim \mathcal{N}(\mathbf{0}, \Sigma_1)$. This section checks via Monte Carlo simulations the probability of correct model selection~\eqref{eqn:GMM:model:selection} calculated with the decoded signals $\tilde{\bx}_1 = \Delta_1 \Phi \bx$ and $\tilde{\bx}_2 = \Delta_2 \Phi \bx$,
\begin{equation}
\label{eqn:prob:assignment:error}
P_c = (E_\Phi) \left(\int_{\tilde{\bx}_1^T \Sigma_1^{-1} \tilde{\bx}_1 < \tilde{\bx}_2^T \Sigma_2^{-1} \tilde{\bx}_2 } f_1(\bx) d\bx\right) = (E_\Phi) \left( \int \textrm{sign}\left( \tilde{\bx}_2^T \Sigma_2^{-1} \tilde{\bx}_2 - \tilde{\bx}_1^T \Sigma_1^{-1} \tilde{\bx}_1 \right) f_1(\bx) d\bx \right),
\end{equation}
where the expectation is with respect to $\Phi$ if one random $\Phi$ is independently drawn for each $\bx$. We also investigate the MSE of the resulting signal reconstruction,
\begin{equation}
\label{eqn:MSE:SCS:2gaussians}
E_{\bx, (\Phi)}\|\bx - \tilde{\bx}\|^2_2 = (E_\Phi) \left(\int_{\tilde{\bx}_1^T \Sigma_1^{-1} \tilde{\bx}_1 < \tilde{\bx}_2^T \Sigma_2^{-1} \tilde{\bx}_2 } \|\bx - \tilde{\bx}_1\|^2_2  f_1(\bx) d\bx + \int_{\tilde{\bx}_1^T \Sigma_1^{-1} \tilde{\bx}_1 \geq \tilde{\bx}_2^T \Sigma_2^{-1} \tilde{\bx}_2 } \|\bx - \tilde{\bx}_2\|^2_2 f_1(\bx) d\bx \right),
\end{equation}
as a function of the number of sensing measurements $M$ and the properties of the Gaussian distributions.  

Figure~\ref{fig:prob:error:MAP} shows the probability of correct model selection $P_c$~\eqref{eqn:prob:assignment:error} and the MSE of signal reconstruction~\eqref{eqn:MSE:SCS:2gaussians} as a function of the number of measurements $M$ and the signal dimension $N$. Figure~\ref{fig:prob:error:MAP}-(a) plots $P_c$ as a function of $M$ going from $1$ to $N$, with different $N$ values from 2 to 15, assuming that~\eqref{eqn:KL:greedy:optimal:C} holds, i.e., the two Gaussians are ``orthogonal.'' A power decay model of the eigenvalues~\eqref{eqn:eigvalue:power:decay} with a typical decay parameter $\alpha=3$ is assumed in the simulations.  A random Gaussian matrix realization $\Phi$ is drawn independently to sense each  signal. As expected, $P_c$ increases as $M$ goes from 1 to $N$, i.e., as more measurements are dedicated. The signal dimension $N$ plays an important role. With only $M=1$ measurement, the model selection is uniformly random ($P_c \approx 0.5$), independent of the signal dimensions $N$. At an extremely low dimension $N=2$, even with $M=N$ measurements (which leads to perfect signal reconstruction, as if in the ``oracle'' case described in Section~\ref{subsec:oracle:model:selection}), $P_c$ remains lower than 0.8.~\footnote{We observe that a mistake in the model selection will not necessarily lead to a mistake in the reconstruction, e.g., flat
image patches can often be recovered by multiple different models.} When $N$ goes higher, $P_c$  rapidly increases converging towards 1 as $M$ increases. After $N$ stands above a certain value (about 10 in this example, note that for the image examples in the next section $N=64$), $P_c$ converges very close to 1 as far as $M$ reaches a fixed value (about 8) independent of $N$. This indicates that accurate model selection can be achieved with very low sampling rates $M/N$, given that the energy of the signals is concentrated in the first few principal dimensions. 
In signal sampling, one is more interested in the signal reconstruction error than model selection. Figure~\ref{fig:prob:error:MAP}-(b) similarly shows the MSE of the decoded signals~\eqref{eqn:MSE:SCS:2gaussians} (normalized by the ideal signal energy). The MSE decreases as $M$ increases, and it goes to 0 as $M=N$. At high dimensions $N$ (over about 10), almost perfect signal reconstruction is obtained as far as $M$ reaches a fixed value (about 8). 

\begin{figure}[htbp]
\vspace{-13ex}
\begin{center}
\begin{tabular}{cc}
\includegraphics[width=6cm]{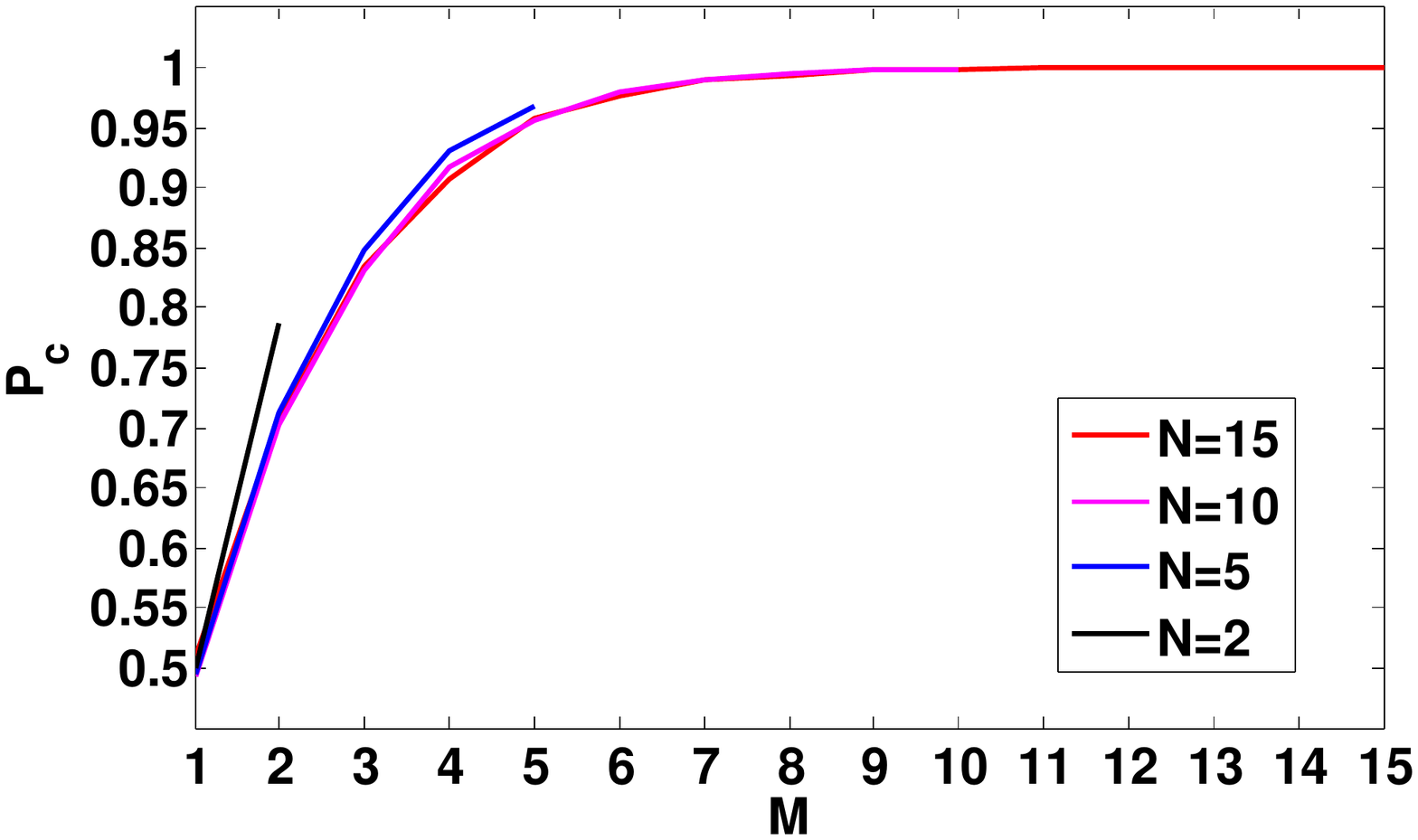}  &
\includegraphics[width=6cm]{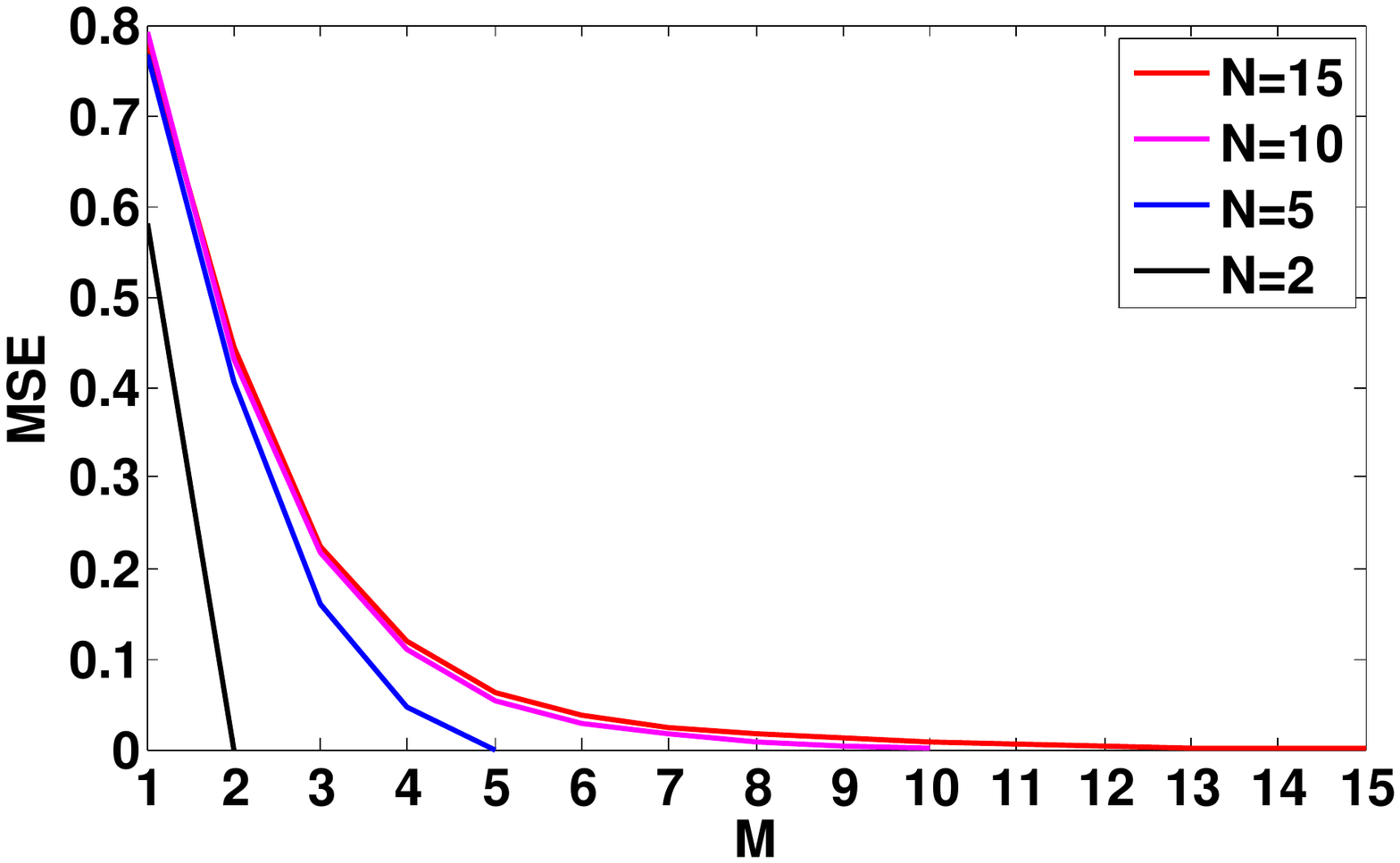} \vspace{-13ex} \\
\textbf{\small(a)} & \textbf{\small(b)} \vspace{-1ex} \\
\end{tabular}
\end{center}
\vspace{0ex}
\caption{(a) The probability of correct model selection~\eqref{eqn:prob:assignment:error} as a function the of the number of measurements $M$ from $1$ to the signal dimension $N$, with $N$ going from 2 to 15.
(b) The same for MSE~\eqref{eqn:MSE:SCS:2gaussians} (normalized by the ideal signal energy) of the decoded signals.} \label{fig:prob:error:MAP}
\vspace{0ex}
\end{figure}

Similarly, Figure~\ref{fig:prob:error:MAP:alpha} plots the probability of correct model selection $P_c$~\eqref{eqn:prob:assignment:error} as well as the MSE of the decoded signals~\eqref{eqn:MSE:SCS:2gaussians} (normalized by the ideal signal energy), as a function of the measurements $M$ going from 1 to the signal dimension $N=10$, with different eigenvalue decay parameter $\alpha$ from 1 to 5. As $\alpha$ increases, i.e., as the eigenvalues decay faster, $P_c$ and MSE respectively converge to 1 and 0 at a faster rate as $M$ goes from 1 to $N$.

\begin{figure}[htbp]
\vspace{-13ex}
\begin{center}
\begin{tabular}{cc}
\includegraphics[width=6cm]{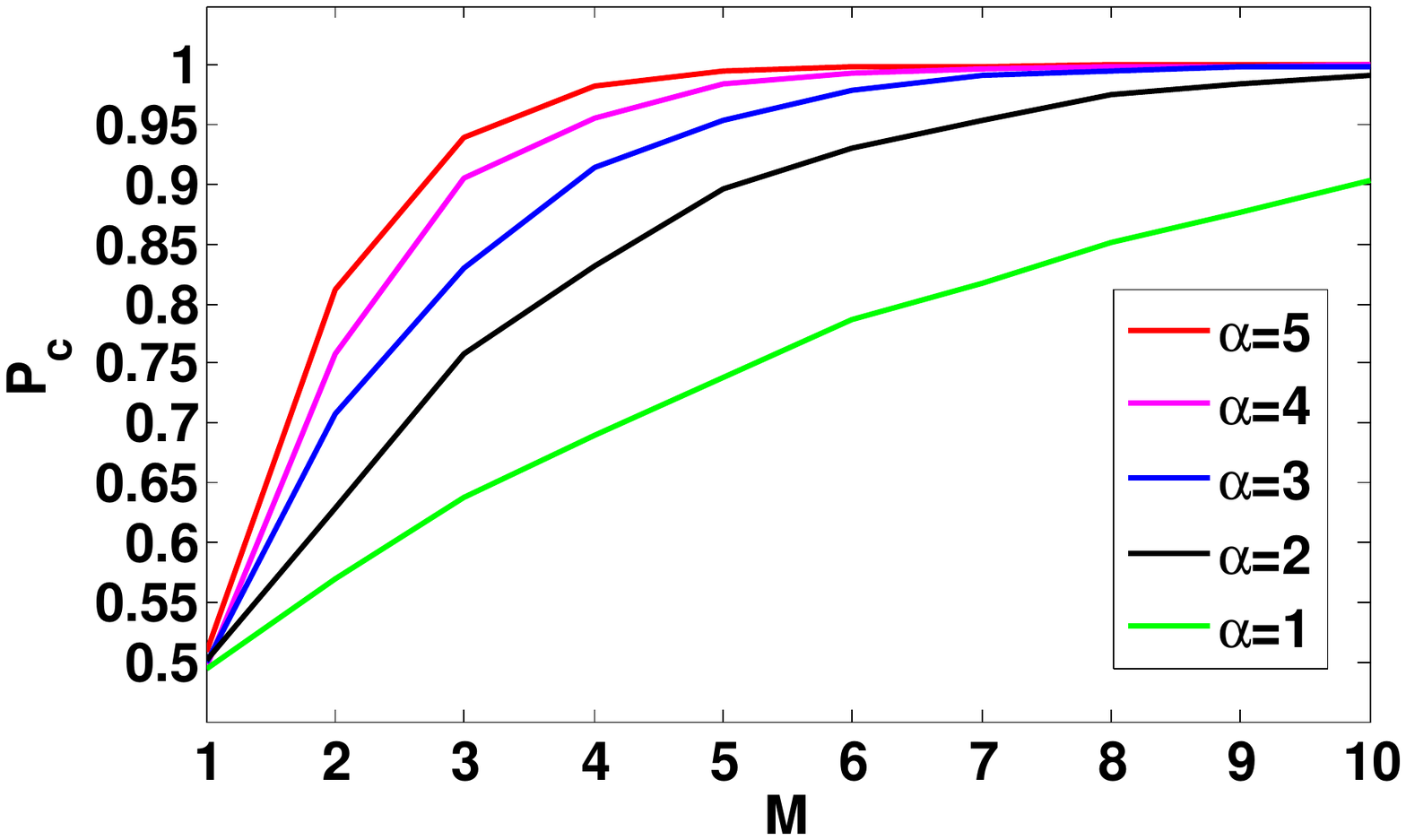}  &
\includegraphics[width=6cm]{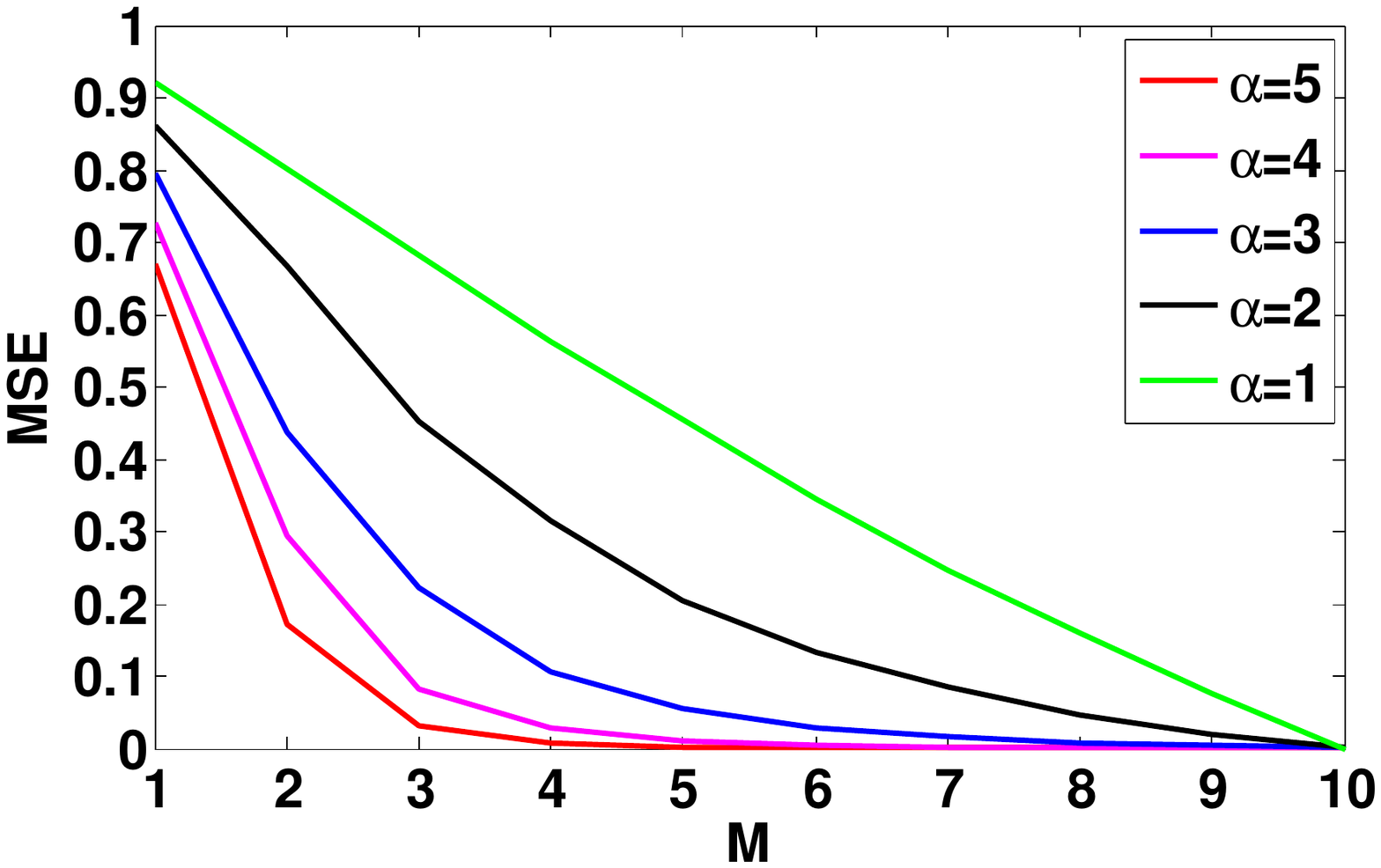} \vspace{-13ex} \\
\textbf{\small(a)} & \textbf{\small(b)} \vspace{-1ex} \\
\end{tabular}
\end{center}
\vspace{0ex}
\caption{(a) The probability of correct model selection~\eqref{eqn:prob:assignment:error}, as a function of the number of measurements $M$ from $1$ to the signal dimension $N=10$, with different eigenvalue decay parameter $\alpha$ from 1 to 5. (b) The same for MSE~\eqref{eqn:MSE:SCS:2gaussians} (normalized by the ideal signal energy) of the decoded signals.} \label{fig:prob:error:MAP:alpha}
\vspace{0ex}
\end{figure}

In summary, this section shows that the accuracy of the Gaussian model selection~\eqref{eqn:GMM:model:selection} in GMM-based SCS is influenced by a number of factors including the geometry
of the Gaussian distributions in the GMM, the signal dimension, and the number of sensing measurements. More accurate model selection is obtained as the Gaussians distributions are more ``orthogonal'' one another, as each of the Gaussians is more anisotropic, as the signals are in a higher dimension given that the energy of the signals are concentrated in the first few dimensions, and as the number of sensing measurements increases.

\section{SCS with GMM -- Algorithm and Experiments}
\label{sec:SCS:GMM:algorithm}

The GMM-based SCS decoder described in Section~\ref{subsec:algorithm:SCS:GMM} assumes that the means and the covariances of the Gaussian distributions $\{\mathcal{N} (\mu_j, \Sigma_j)\}_{1 \leq j \leq J}$ in the GMMs are known. However, in real sensing applications, these parameters are unavailable. Following~\cite{yu2010PLE}, this ection presents a maximum a posteriori expectation-maximization (MAP-EM) algorithm~\cite{allassonniere2007towards} that iteratively estimates the Gaussian parameters and decodes the signals. GMM-based SCS calculated with the MAP-EM algorithm is applied in real signal sensing, and is compared with conventional CS based on sparse models. 

\subsection{MAP-EM Algorithm}

The MAP-EM algorithm is an iterative procedure that alternates between two steps: 
\subsubsection{E-step} Assuming that the estimates of the Gaussian parameters $\{(\tilde{\mu}_j, \tilde{\Sigma}_j)\}_{1 \leq j \leq J}$ are known (following the previous M-step), the E-step calculates the MAP signal estimation and model selection for all the signals, following~\eqref{eqn:MAP:x:k}--\eqref{eqn:best:decoder} . 
\subsubsection{M-step} Assuming that the Gaussian model selection $\tilde{j}$ and the signal estimate $\tilde{\bx}$ are known for all the signals (following the previous E-step), the M-step estimates (updates) the Gaussian models $\{(\tilde{\mu}_j, \tilde{\Sigma}_j)\}_{1 \leq j \leq J}$. 

Let $\bx_i$, $\by_i$,  $\tilde{\bx}_i$ and $\tilde{j}_i$ respectively denote the $i$-th signal in the collection, its coded version, its estimate, and its estimated Gaussian model index, $1 \leq i \leq I$. Let $\mathcal{C}_j$ be the ensemble of the signal indices $i$ that are assigned to the $k$-th Gaussian model, i.e., $\mathcal{C}_j = \{i: \tilde{j}_i=j\}$, and let $|\mathcal{C}_j|$ be its cardinality. The parameters of each Gaussian model are estimated with the maximum likelihood estimate using all the signals assigned to that Gaussian model,
\begin{equation}
\label{eqn:ML:gaussian}
(\tilde{\mu}_j, \tilde{\Sigma}_j) = \arg \max_{\mu_j, \Sigma_j} \log f(\{\tilde{\bx}_i\}_{i \in \mathcal{C}_j}|\mu_j, \Sigma_j). 
\end{equation}
With the Gaussian model~\eqref{eqn:multivariate:gaussian} , it is well-known that the resulting estimate is the empirical estimate
\begin{equation}
\label{eqn:ML:covariacne}
\tilde{\mu}_j = \frac{1}{|\mathcal{C}_j|}\sum_{i \in \mathcal{C}_j} \tilde{\bx}_i~~\textrm{and}~~
\tilde{\Sigma}_j = \frac{1}{|\mathcal{C}_j|} \sum_{i \in \mathcal{C}_j} (\tilde{\bx}_i  - \tilde{\mu}_j) (\tilde{\bx}_i  - \tilde{\mu}_j)^T. 
\end{equation}

The computational complexity of the MAP-EM algorithm is dominated by the matrix inversion $(\Phi \Sigma_j \Phi^T)^{-1}$ in~\eqref{eqn:MAP:Sigma:K} in the E-step. It can be implemented with $M^3/3$ flops through a Cholesky factorization~\cite{boyd2004convex}. With $J$ Gaussian models, the complexity per iteration is therefore dominated by $JM^3/3$ flops. 

As the MAP-EM algorithm described above iterates, the MAP probability of the observed signals $f(\{\tilde{\bx}_i\}_{1 \leq i \leq I} | \{\by_i\}_{1 \leq i \leq I}, \{\tilde{\mu}_j, \tilde{\Sigma}_j\}_{1 \leq j \leq J})$ always increases. This can be observed by interpreting the E- and M-steps as a coordinate descent optimization~\cite{hathaway1986another}.  

The algorithm initialization and the number $J$ of Gaussians in GMM can be selected according to the type of signals of interest. For sensing natural images, a geometry-motivated initialization as detailed in~\cite{yu2010PLE} will be applied in the experiments. 

\subsection{Experiments}
The GMM-based SCS is applied in real image sensing, and is compared with conventional CS based on sparse models. Following standard practice, an image is decomposed into $\sqrt{N} \times \sqrt{N} = 8 \times 8$ local patches $\{\bx_i\}_{1 \leq i \leq I}$ (an image patch is reshaped to and considered as a vector)~\cite{aharon2006k, mairal2008sparse,yu2010PLE}, which are assumed to follow a GMM~\cite{yu2010PLE}. SCS samples each patch $\by_i = \Phi_i \bx_i$, with a possibly different $\Phi_i$ for each $\bx_i$. The decoder is implemented with the MAP-EM algorithm, initialized with $J=19$ geometry-motivated Gaussian models, each capturing a local direction~\cite{yu2010PLE}. The algorithm typically converges within 3 iterations. No database is used, and all the parameters and reconstruction are learned from the compressed sensed image alone.

The dictionary for conventional CS is learned with K-SVD~\cite{aharon2006k} from 720,000 image patches, extracted from the entire standard Berkeley segmentation database containing 300 natural images~\cite{MartinFTM01}. In image estimation and sensing, learned dictionaries have been shown to produce better results than off-the-shelf ones~\cite{aharon2006k, duarte2009learning,mairal2008sparse}. The decoder is calculated with the $l_1$ minimization~\cite{tibshirani1996regression} implemented in~\cite{mairal2009online}. 
Three standard images Lena ($512 \times 512$), House ($256 \times 256$), and Peppers ($512 \times 512$), as illustrated in Figure~\ref{fig:standard:images}, are used in the experiments. 

\begin{figure}[htbp]
\vspace{0ex}
\begin{center}
\begin{tabular}{ccc}
\hspace{0ex}
\includegraphics[width=3cm]{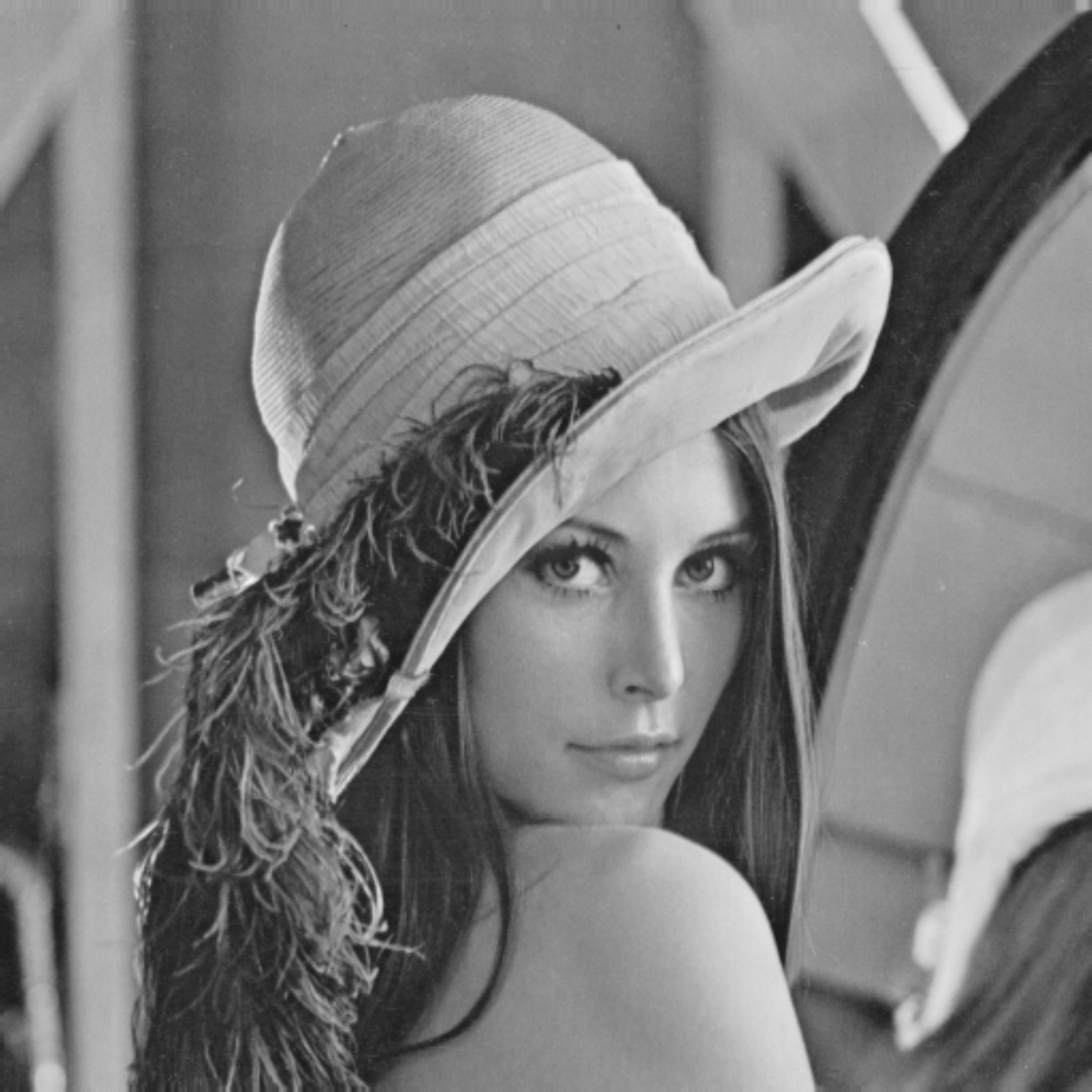}  
\includegraphics[width=3cm]{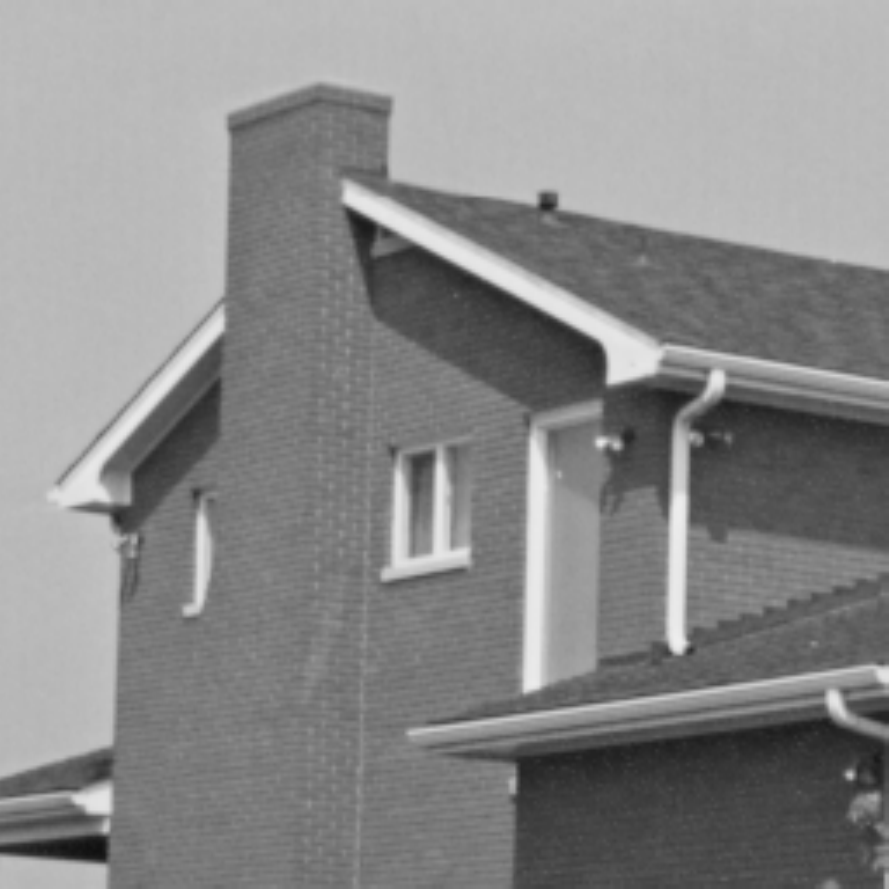}  
\includegraphics[width=3cm]{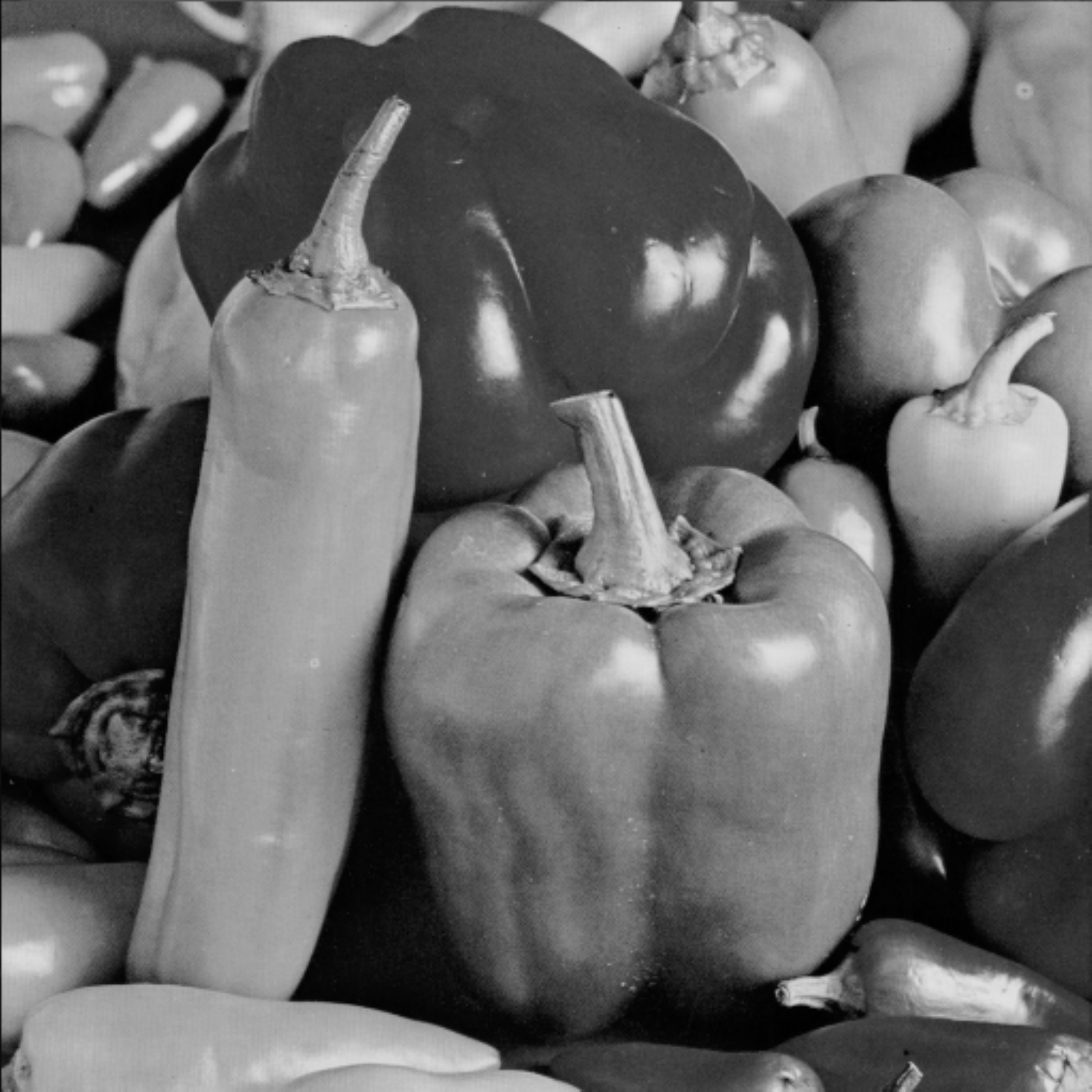}  \\
\end{tabular}
\end{center}
\vspace{0ex}
\caption{From left to right. Three standard images used for the experiments: Lena, House, and Peppers.} \label{fig:standard:images} 
\vspace{0ex}
\end{figure}

\begin{figure}[htbp]
\vspace{0ex}
\begin{center}
\begin{tabular}{cc}
\hspace{0ex}
\includegraphics[width=7cm]{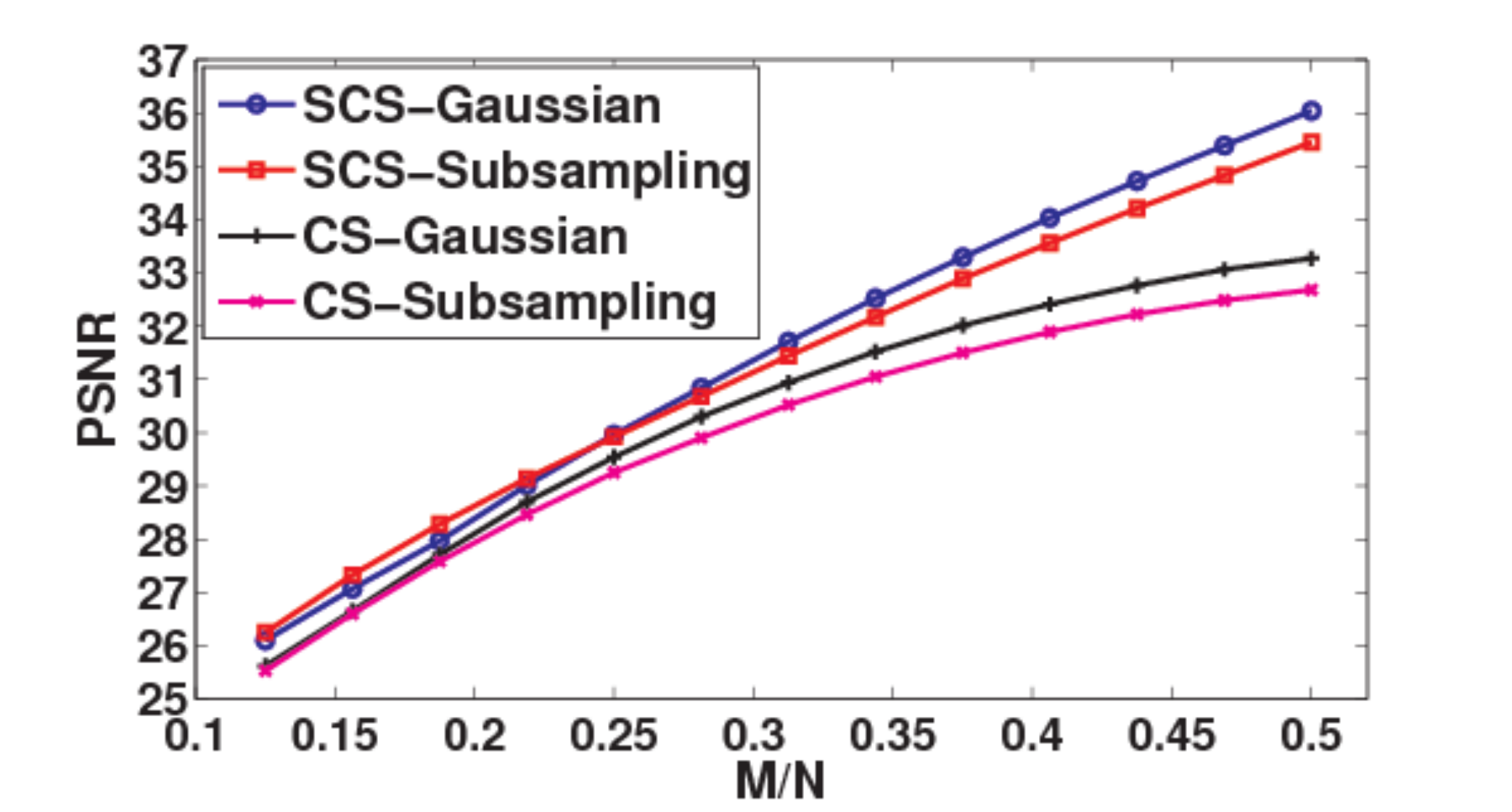}  \hspace{0ex} &
\includegraphics[width=7.35cm]{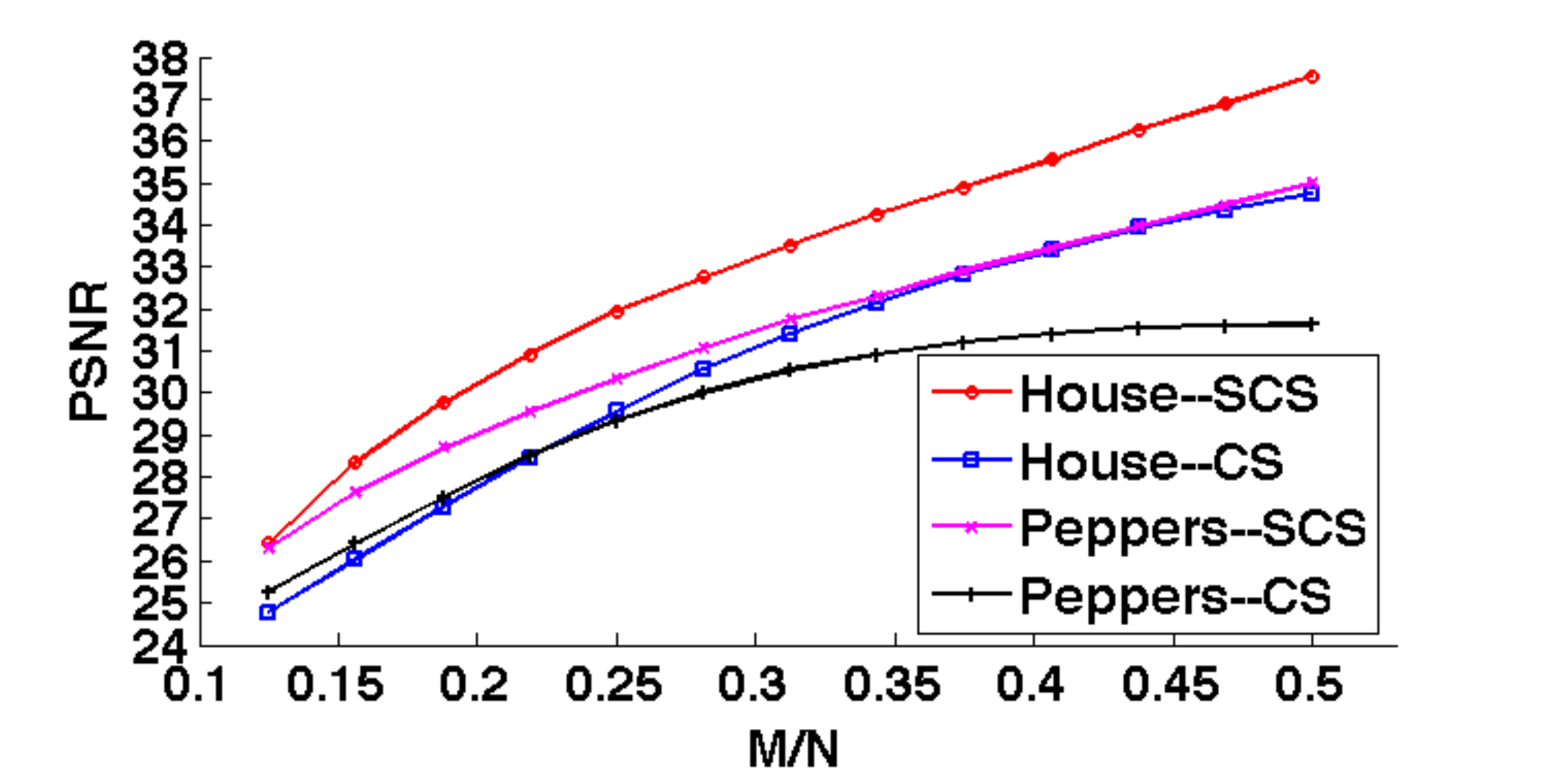} \vspace{0ex} \\
\textbf{\small(a)} & \textbf{\small(b)} \vspace{0ex} \\
\end{tabular}
\end{center}
\vspace{0ex}
\caption{(a) PSNR (dB) vs sampling rate for SCS and CS using Gaussian and random subsampling sensing matrices on image patches extracted from Lena. (b) PSNR (dB) vs sampling rate for SCS and CS using Gaussian sensing matrices on image patches extracted from House and Peppers. } \label{fig:experiments}
\vspace{0ex}
\end{figure}

Figure~\ref{fig:experiments} (a) shows the sensing performance on about 260,000 (sliding) patches, regarded as signals $\bx_i$, extracted from Lena. The PSNRs generated by SCS and CS using Gaussian and random subsampling sensing matrices, one independent realization for each patch, are plotted as a function of the sampling rate $M/N$. At the same sampling rate, SCS outperforms SC. The gain increases from about 0.5 dB at very low sampling rates ($M/N \approx 0.1$), learning a GMM from the poor-quality measured data being more challenging,  to more than 3.5 dB at high sampling rates ($M/N \approx 0.5$). (SC using an ``oracle'' dictionary learned from the ideal Lena itself, undoable in practice, improves its performance from 0.2 dB at low sampling rates to 1.3 dB at high sample rates, still lower than SCS.) For both SCS and CS, Gaussian and random subsampling matrices lead to similar PSNRs at low sampling rates ($M/N<0.25$), and at higher sampling rates Gaussian sensing gains by about 0.5 dB. Recall that SCS is not just more accurate and significantly faster, but also uses only the compressed image, while conventional CS uses a pre-learned dictionary from a large database.

Figure~\ref{fig:experiments} (b) further compares SCS with CS on sliding patches, regarded as signals, extracted from Peppers (260,000 patches) and House (62,000 patches). One independent Gaussian matrix realization is applied to sense each patch. Similar results as on the patches from Lena  are observed. At the same sampling rate, SCS outperforms SC. The gain is smaller (about 1 dB) at very low sampling rates ($M/N \approx 0.1$), and becomes substantial (about 3 dB) at high sampling rates ($M/N \approx 0.5$).

Figure~\ref{fig:patches} illustrates some typical patches with geometry. The ground-truth patches are shown in the first row, and the patches reconstructed by conventional CS and SCS, all sensed with Gaussian matrices at a sampling rate $M/N=1/4$, are respectively illustrated in the second and the third row. Both CS and SCS lead to accurate reconstruction in uniform regions. SCS outperforms CS on the more geometrical parts, and the improvement is significant on the fine contours (the 2nd, 3rd and 7th patches). 

\begin{figure}[htbp]
\vspace{0ex}
\begin{center}
\begin{tabular}{c}
\includegraphics[width=12cm]{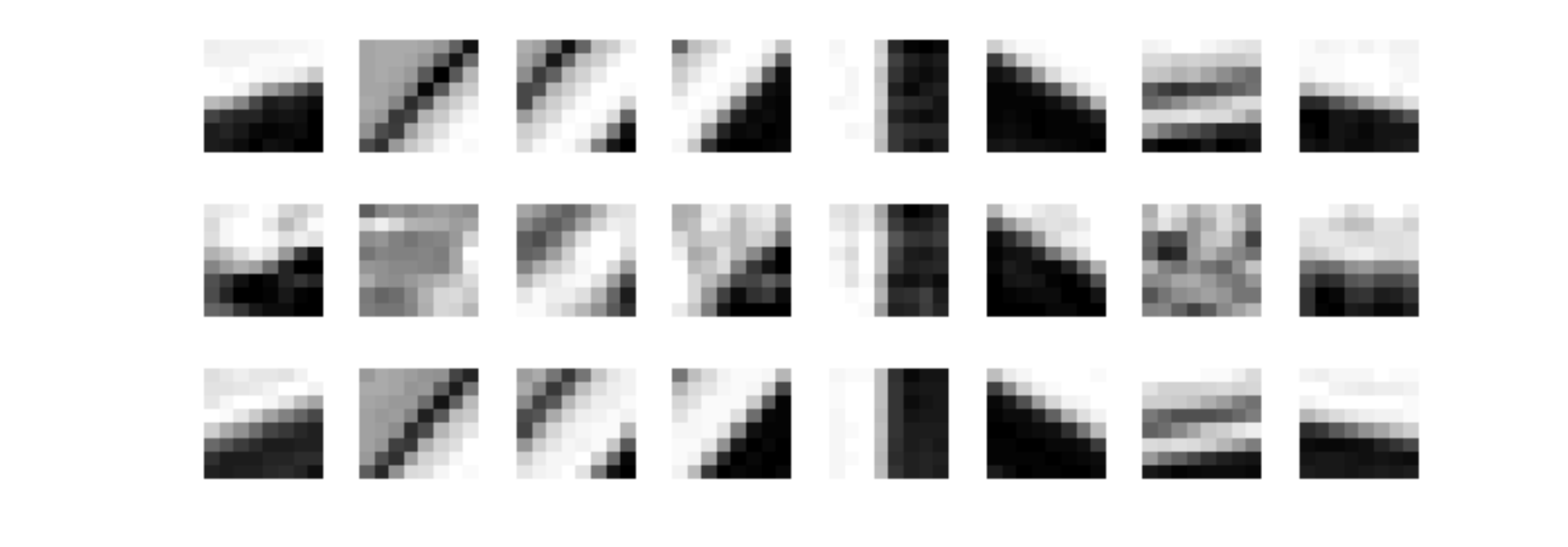}  
 \vspace{0ex} 
\end{tabular}
\end{center}
\vspace{0ex}
\caption{Some typical $8 \times 8$ patches with geometry. First row: ground-truth patches. Second and third rows: patches reconstructed by conventional CS and SCS respectively, all sensed at a sampling rate $M/N=1/4$ with Gaussian matrices.} \label{fig:patches}
\vspace{0ex}
\end{figure}

In most image sensing applications, one is interested in reconstructing whole images instead of individual patches. Aggregating~\textit{non-overlapped}~patches to a whole images produces block artifacts, as illustrated in Figure~\ref{fig:Lena:block}. It is well known that averaging \textit{overlapped} reconstructed patches not only removes the block artifacts, but also considerably improves the image estimation~\cite{aharon2006k, mairal2008sparse,yu2010PLE}. However, compressed sensing only allows sensing \textit{non-overlapping} patches, since sensing \textit{overlapping} patches would dramatically increase the sampling rate. Nevertheless, \textit{overlapped} reconstructed patches are computable if the sensing operators, performed on \textit{non-overlapped} patches, are random subsampling matrices, which are diagonal operators (one non-zero entry per row). (The reconstruction is then equivalent to solving an inpainting problem~\cite{aharon2006k, yu2010PLE}.)  Figure~\ref{fig:Lena:block} shows some typical regions in Lena. The overlapped reconstruction, which further supports the search for performance on average as in the proposed SCS, removes the block artifacts and significantly improves the reconstructed image. Figure~\ref{fig:sliding} plots the PSNRs on the whole image Lena generated by SCS using random subsampling matrices and overlapped reconstruction are plotted, in comparison with those obtained using Gaussian sensing matrices and non-overlapped reconstruction, at different sampling rates. The former improves from about 3.5 dB, at low sampling rates, to 1.5 dB, at high sampling rates, at a cost of $N=64$ times computation. 

\begin{figure}[htbp]
\vspace{0ex}
\begin{center}
\begin{tabular}{ccc}
\hspace{0ex}\includegraphics[width=3cm]{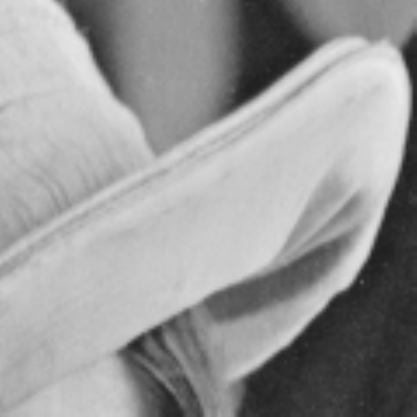}  &
\hspace{0ex}\includegraphics[width=3cm]{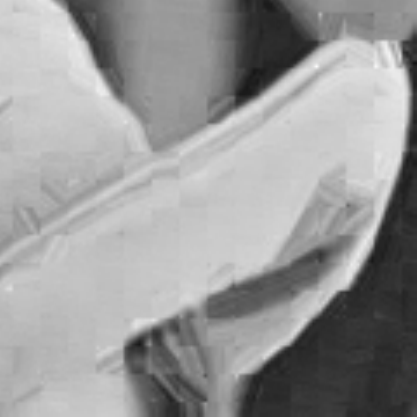} &
\hspace{0ex}\includegraphics[width=3cm]{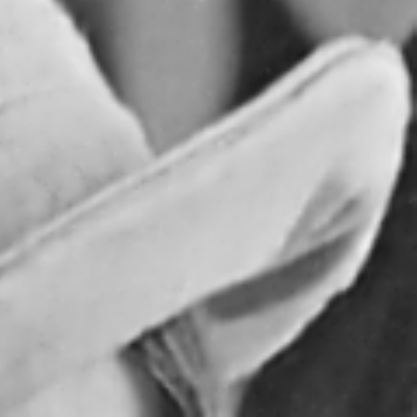}  \\
\hspace{0ex}{\textbf{Ground truth}}
&\hspace{0ex} {\textbf{No.-ovl. rec. 30.82 dB}}&\hspace{0ex} {\textbf{Ovl. rec. 34.02 dB}} \vspace{1ex}\\
\hspace{0ex}\includegraphics[width=3cm]{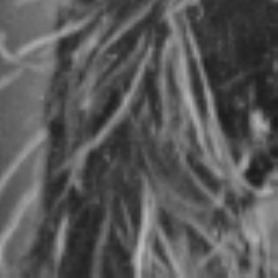}  &
\hspace{0ex}\includegraphics[width=3cm]{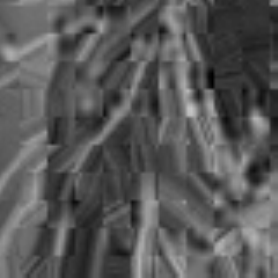}  &
\hspace{0ex}\includegraphics[width=3cm]{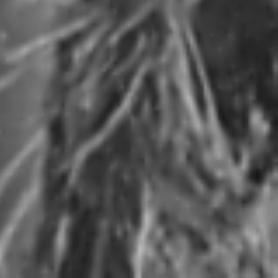}  \\
\hspace{0ex}{\textbf{Ground truth}}
&\hspace{0ex} {\textbf{No.-ovl. rec. 24.72 dB}}&\hspace{0ex} {\textbf{Ovl. rec. 27.87 dB}} \\
\end{tabular}
\end{center}
\vspace{0ex}
\caption{From left to right. Zoomed crops from Lena, reconstructed images by SCS using Gaussian sensing matrices and non-overlapping reconstruction, and by SCS using subsampling random matrices and overlapping reconstruction. The image is sensed on \textit{non-overlapped} patches at a sampling rate of $M/N=0.25$. Local PSNRs are reported.} \label{fig:Lena:block}
\vspace{0ex}
\end{figure}

\begin{figure}[htbp]
\vspace{0ex}
\begin{center}
\begin{tabular}{c}
\includegraphics[width=8cm]{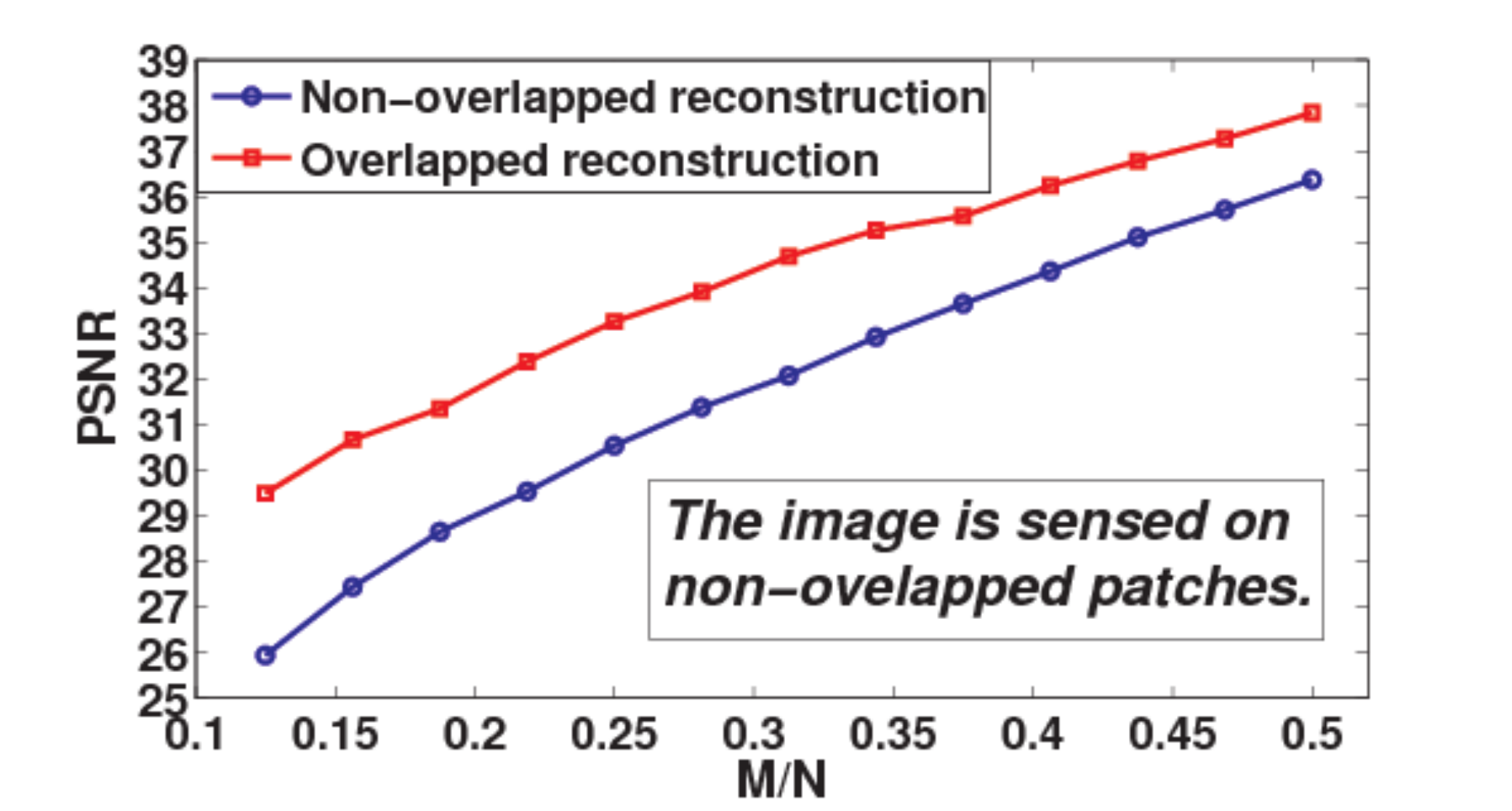}  
 \vspace{0ex} 
\end{tabular}
\end{center}
\vspace{0ex}
\caption{PSNR (dB) vs sampling rate (on the whole image Lena), for SCS using Gaussian sensing matrices with non-overlapping reconstruction, and subsampling random matrices with overlapping reconstruction.} \label{fig:sliding}
\vspace{0ex}
\end{figure}

\section{Conclusion}
Statistical compressed sensing (SCS) based on statistical signal models has been introduced. As opposed to conventional compressed sensing that aims at efficiently sensing and accurately reconstructing one signal at a time, SCS deals simultaneously with a collection of signals. While CS assumes signal sparse models, SCS is based on a more general Bayesian assumption that signals follow a statistical distribution. SCS based on Gaussian models has been investigated in depth. It has been shown that based on a single Gaussian model, with Gaussian or Bernoulli sensing matrices of $\mathcal{O}(k)$ measurements, considerably smaller than the $\mathcal{O}(k \log(N/k))$ required by conventional CS, where $N$ is the signal dimension, and with an optimal decoder implemented with linear filtering, significantly faster than the pursuit decoders applied in conventional CS, the error of SCS is tightly upper bounded by a constant times the best $k$-term approximation error, with overwhelming probability. The failure probability is also significantly smaller than that of conventional CS.  Stronger yet simpler results, derived from a new RIP in expectation property further show that for \textit{any} sensing matrix, the error of Gaussian SCS is upper bounded by a constant times the best $k$-term approximation with probability one, and the bound constant can be efficiently calculated. For Gaussian mixture models (GMMs) that assume multiple Gaussian distributions, and that each signal follows one of them with an unknown index, a piecewise linear estimator is introduced to decode SCS. The accuracy of model selection, which is at the heart of the piecewise linear decoder, is analyzed in terms of the properties of the Gaussian distributions and the number of the sensing measurements. A MAP-EM algorithm that iteratively estimates the Gaussian models and decodes the compressed signals is presented for GMM-based SCS. Applications of GMM-based SCS in real image sensing has been shown. Comparing with conventional CS, SCS leads to improved results, at a considerably lower computational cost. 

This line of research opens numerous new questions in compressed sensing, from the formal development of bounds in the compressed domain model selection (see also~\cite{Calderbank2011}), to the study of model parameters estimation in the compressed domain and the extension of the results here reported to non-Gaussian distributions. The work here reported also shows that compressed sensing is significant beyond sparse signal models, generating the natural question of what type of models can benefit from such sensing scenario.

\vspace{2ex}
\noindent {\it \textbf{Acknowledgements:} Work partially supported by NSF, ONR, NGA, ARO, DARPA, and NSSEFF. The authors thank very much St\'ephane Mallat for co-developing the GMM framework reported in~\cite{yu2010PLE} for solving inverse problems.}

\bibliographystyle{plain}
\bibliography{biblio}

\begin{thebibliography}{10}

\bibitem{achlioptas2003database}
D.~Achlioptas.
\newblock {Database-friendly random projections: Johnson-Lindenstrauss with
  binary coins}.
\newblock {\em Journal of Computer and System Sciences}, 66(4):671--687, 2003.

\bibitem{aharon2006k}
M.~Aharon, M.~Elad, and A.~Bruckstein.
\newblock {K-SVD: An algorithm for designing overcomplete dictionaries for
  sparse representation}.
\newblock {\em IEEE Trans. on Signal Proc.}, 54(11):4311, 2006.

\bibitem{allassonniere2007towards}
S.~Allassonniere, Y.~Amit, and A.~Trouv{\'e}.
\newblock {Towards a coherent statistical framework for dense deformable
  template estimation}.
\newblock {\em J.R. Statist. Soc. B}, 69(1):3--29, 2007.

\bibitem{baraniuk2008simple}
R.~Baraniuk, M.~Davenport, R.~DeVore, and M.~Wakin.
\newblock {A simple proof of the restricted isometry property for random
  matrices}.
\newblock {\em Constructive Approximation}, 28(3):253--263, 2008.

\bibitem{baraniuk2010model}
R.G. Baraniuk, V.~Cevher, M.F. Duarte, and C.~Hegde.
\newblock {Model-based compressive sensing}.
\newblock {\em IEEE Trans. on Info. Theo.}, 56(4):1982--2001, 2010.

\bibitem{baraniuk2009random}
R.G. Baraniuk and M.B. Wakin.
\newblock {Random projections of smooth manifolds}.
\newblock {\em Foundations of Comp. Math.}, 9(1):51--77, 2009.

\bibitem{boyd2004convex}
S.P. Boyd and L.~Vandenberghe.
\newblock {\em {Convex Optimization}}.
\newblock Cambridge University Press, 2004.

\bibitem{buades2006review}
A.~Buades, B.~Coll, and J.M. Morel.
\newblock {A review of image denoising algorithms, with a new one}.
\newblock {\em Multiscale Modeling and Simulation}, 4(2):490--530, 2006.

\bibitem{Calderbank2011}
R.~Calderbank, S.~Jafarpour, and Kent. J.
\newblock {Finding needles in compressed haystacks}.
\newblock {\em in Compressed Sensing, Y. Eldar and G. Kutynok, Eds., Cambridge
  University Press, to appear}, 2011.

\bibitem{candes2007sparsity}
E.~Cand{\`e}s and J.~Romberg.
\newblock {Sparsity and incoherence in compressive sampling}.
\newblock {\em Inverse Problems}, 23:969, 2007.

\bibitem{candes2010compressed}
E.J. Candes, Y.C. Eldar, D.~Needell, , and P.~Randall.
\newblock {Compressed sensing with coherent and redundant dictionaries}.
\newblock {\em In Press, Applied and Computational Harmonic Analysis}, 2011.

\bibitem{candes2006robust}
E.J. Cand{\`e}s, J.~Romberg, and T.~Tao.
\newblock {Robust uncertainty principles: Exact signal reconstruction from
  highly incomplete frequency information}.
\newblock {\em IEEE Trans. on Info. Theo.}, 52(2):489--509, 2006.

\bibitem{candes2005decoding}
E.J. Cand\`es and T.~Tao.
\newblock {Decoding by linear programming}.
\newblock {\em IEEE Trans. on Info. Theo.}, 51(12):4203--4215, 2005.

\bibitem{candes2006near}
E.J. Cand\`es and T.~Tao.
\newblock {Near-optimal signal recovery from random projections: Universal
  encoding strategies?}
\newblock {\em IEEE Trans. on Info. Theo.}, 52(12):5406--5425, 2006.

\bibitem{chechik2007max}
G.~Chechik, G.~Heitz, G.~Elidan, P.~Abbeel, and D.~Koller.
\newblock {Max-margin classification of incomplete data}.
\newblock In {\em Proceedings of the Advances in Neural Information Processing
  Systems}, page 233, 2007.

\bibitem{chen2010compressive}
M.~Chen, J.~Silva, J.~Paisley, C.~Wanng, D.~Dunson, and L.~Carin.
\newblock {Compressive sensing on manifolds using a nonparametric mixture of
  factor analyzers: Algorithm and performance bounds}.
\newblock {\em IEEE Trans. on Signal Proc.}, 58(12):6140--6155, 2010.

\bibitem{cohen2009compressed}
A.~Cohen, W.~Dahmen, and R.~DeVore.
\newblock {Compressed sensing and best k-term approximation}.
\newblock {\em J. of the Am, Math. Soc.}, 22(1):211--231, 2009.

\bibitem{donoho2006compressed}
D.L. Donoho.
\newblock {Compressed sensing}.
\newblock {\em IEEE Trans. on Info. Theo.}, 52, 2006.

\bibitem{duarte2009learning}
J.M. Duarte-Carvajalino and G.~Sapiro.
\newblock {Learning to sense sparse signals: Simultaneous sensing matrix and
  sparsifying dictionary optimization}.
\newblock {\em IEEE Transactions on Image Processing}, 18(7):1395--1408, 2009.

\bibitem{duda2000pattern}
R.O. Duda, P.E. Hart, and D.G. Stork.
\newblock {\em {Pattern Classification}}.
\newblock Wiley-Interscience, 2000.

\bibitem{eldar2009robust}
Y.C. Eldar and M.~Mishali.
\newblock {Robust recovery of signals from a structured union of subspaces}.
\newblock {\em IEEE Trans. on Info. Theo.}, 55(11):5302--5316, 2009.

\bibitem{hathaway1986another}
R.J. Hathaway.
\newblock {Another interpretation of the EM algorithm for mixture
  distributions}.
\newblock {\em Statistics \& Probability Letters}, 4(2):53--56, 1986.

\bibitem{kay1998fundamentals}
S.M. Kay.
\newblock {\em {Fundamentals of Statistical Signal Processing, Volume 1:
  Estimation Theory}}.
\newblock Prentice Hall, 1998.

\bibitem{leger2010Matrix}
F.~L\'eger, G.~Yu, and G~Sapiro.
\newblock {Efficient matrix completion with Gaussian models}.
\newblock {\em Submitted, arxiv.org/abs/1010.4050}, Oct., 2010.

\bibitem{mairal2009online}
J.~Mairal, F.~Bach, J.~Ponce, and G.~Sapiro.
\newblock {Online dictionary learning for sparse coding}.
\newblock In {\em Proceedings of the 26th Annual International Conference on
  Machine Learning}, pages 689--696. ACM, 2009.

\bibitem{mairal2008sparse}
J.~Mairal, M.~Elad, and G.~Sapiro.
\newblock {Sparse representation for color image restoration}.
\newblock {\em IEEE Trans. on Image Proc.}, 17, 2008.

\bibitem{mallat2008wts}
S.~Mallat.
\newblock {\em {A Wavelet Tour of Signal Processing: The Sparse Way, 3rd
  Ddition}}.
\newblock Academic Press, 2008.

\bibitem{MartinFTM01}
D.~Martin, C.~Fowlkes, D.~Tal, and J.~Malik.
\newblock A database of human segmented natural images and its application to
  evaluating segmentation algorithms and measuring ecological statistics.
\newblock In {\em Proc. ICCV}, 2001.

\bibitem{masiero2010data}
R.~Masiero, G.~Quer, D.~Munaretto, M.~Rossi, J.~Widmer, and M.~Zorzi.
\newblock {Data acquisition through joint compressive sensing and principal
  component analysis}.
\newblock In {\em Global Telecommunications Conference, 2009. GLOBECOM 2009.
  IEEE}, pages 1--6. IEEE, 2009.

\bibitem{petersen2006matrix}
K.B. Petersen and M.S. Pedersen.
\newblock {The matrix cookbook}.
\newblock {\em Technical University of Denmark}, 2006.

\bibitem{peyre2010best}
G.~Peyr{\'e}.
\newblock {Best basis compressed sensing}.
\newblock {\em IEEE Transactions on Signal Processing}, 58(5):2613--2622, 2010.

\bibitem{rother2004grabcut}
C.~Rother, V.~Kolmogorov, and A.~Blake.
\newblock {Grabcut: Interactive foreground extraction using iterated graph
  cuts}.
\newblock {\em ACM Transactions on Graphics (TOG)}, 23(3):309--314, 2004.

\bibitem{talagrand1996new}
M.~Talagrand.
\newblock {A new look at independence}.
\newblock {\em Ann. Prob.}, 24:1, 1996.

\bibitem{tibshirani1996regression}
R.~Tibshirani.
\newblock {Regression shrinkage and selection via the lasso}.
\newblock {\em J. of the Royal Stat. Society}, pages 267--288, 1996.

\bibitem{yu2008audio}
G.~Yu, S.~Mallat, and E.~Bacry.
\newblock {Audio denoising by time-frequency block thresholding}.
\newblock {\em IEEE Trans. on Signal Proc.}, 56(5):1830, 2008.

\bibitem{yu2010PLE}
G.~Yu, G.~Sapiro, and S.~Mallat.
\newblock {Solving inverse problems with piecewise linear estimators: from
  Gaussian mixture models to structured sparsity}.
\newblock {\em Submitted, arxiv.org/abs/1006.3056}, June, 2010.

\end{thebibliography}

\end{document}